\documentclass{article}
\usepackage[a4paper]{geometry}
\usepackage[utf8]{inputenc}
\usepackage[english]{babel}
\addto{\captionsenglish}{%
  
}
\usepackage[toc,acronym,nonumberlist,numberline,shortcuts]{glossaries}
\usepackage[labelfont=bf]{subcaption} 
\DeclareCaptionLabelSeparator{nspace}{\hspace{4mm}} 
\usepackage[labelfont=bf,labelsep=nspace]{caption}
\usepackage{rotating}
\usepackage{placeins} 
\usepackage{wrapfig}    
\usepackage{float} 
\usepackage[authoryear]{natbib}
\setcitestyle{aysep={}} 
\usepackage{mathtools} 
\usepackage{amssymb}
\usepackage{mdwlist} 
\usepackage{paralist}
\usepackage{amsthm}
\usepackage{url}
\usepackage[bookmarks,breaklinks=true]{hyperref}
\hypersetup{colorlinks=true, allcolors=blue, pdfpagemode=UseOutlines, bookmarksdepth=3}
\usepackage{hypcap}
\usepackage[]{algorithm2e}
\usepackage{bbding}
\usepackage[nottoc]{tocbibind}
\newcommand{\figref}[1]{Fig.~\ref{#1}}

\newcommand{\secref}[1]{Sect.~\ref{#1}} 
\newcommand{\eqnref}[1]{Eq.~\ref{#1}} 
\newcommand{\algref}[1]{Algorithm~\ref{#1}} 
\DeclarePairedDelimiter{\abs}{\lvert}{\rvert}
\DeclareMathAlphabet{\mathcal}{OMS}{cmsy}{m}{n}

\mathcode`\*="8000
{\catcode`\*\active\gdef*{\cdot}}

\newacronym{FOR}{f.o.r.}  {frame of reference}
\newacronym{GT}{GT}   {Ground-Truth}
\newacronym{FFT}{FFT} {Fast Fourier Transform}
\newacronym{FT}{FT}   {Fourier Transform}
\newacronym{1D}{1D}   {1-Dimensional}
\newacronym{2D}{2D}   {2-Dimensional}
\newacronym{3D}{3D}   {3-Dimensional}
\newacronym{PAM}{PAM} {Prospecting Asteroid Mission}
\newacronym{LARA}{LARA} {Lander Amorphous Rover Antenna}
\newacronym{MSL}{MSL} {Mean Sea Level}
\newacronym{DOF}{DOF} {Degrees of Freedom}
\newacronym{CG}{CG}   {Center of Gravity}
\newacronym{NED}{NED} {North-East-Down}
\newacronym{ER}{ER}   {Evolutionary Robotics}
\newacronym{AL}{AL}   {Artificial Life}
\newacronym{AI}{AI}   {Artificial Intelligence}
\newacronym{GA}{GA}   {Genetic Algorithm}
\newacronym{FSM}{FSM} {Finite State Machine}
\newacronym{PFSM}{PFSM} {Probabilistic Finite State Machine}
\newacronym{NN}{NN}   {Neural Network}
\newacronym{ANN}{ANN} {Artificial Neural Network}
\newacronym{RNN}{RNN} {Recurrent Neural Network}
\newacronym{FNN}{FNN} {Feed-forward Neural Network}
\newacronym{BT}{BT}   {Behavior Tree}
\newacronym{MDP}{MDP}   {Markov Decision Process}
\newacronym{BDD}{BDD}   {Binary Decision Diagram}
\newacronym{LTL}{LTL}   {Linear Temporal Logic}
\newacronym{CTL}{CTL}   {Computation Tree Logic}
\newacronym{PCTL}{PCTL} {Probabilistic Computation Tree Logic}
\newacronym{RL}{RL}   {Reinforcement Learning}
\newacronym{SMC}{SMC} {Sensory-Motor Coordination}
\newacronym{LS}{LS}   {Least-Squares}
\newacronym{RLS}{RLS} {Recursive Least-Squares}
\newacronym{KF}{KF}   {Kalman Filter}
\newacronym{EKF}{EKF} {Extended Kalman Filter}
\newacronym{UKF}{UKF} {Unscented Kalman Filter}
\newacronym{PF}{PF}   {Particle Filter}
\newacronym{IAEKF}{IAEKF}{Iterative Adaptive Extended Kalman Filter}
\newacronym{KCF}{KCF} {Kalman Consensus Filter}
\newacronym{LPF}{LPF} {Low-Pass Filter}
\newacronym{BPF}{BPF} {Band-Pass Filter}
\newacronym{HPF}{HPF} {High-Pass Filter}
\newacronym{MAF}{MAF} {Moving Average Filter}
\newacronym{CRR}{CRR} {Conflict Resolution Rate}
\newacronym{RMSE}{RMSE} {Root Mean Squared Error}
\newacronym{ZMGN}{ZMGN} {Zero-Mean Gaussian Noise}
\newacronym{GPS}{GPS} {Global Positioning System}
\newacronym{IR}{IR}   {Infra-Red}
\newacronym{IMU}{IMU} {Inertial Measurement Unit}
\newacronym{SLAM}{SLAM} {Simultaneous Localization And Mapping}
\newacronym{AOA}{AOA} {Angle of Arrival}
\newacronym{TOA}{TOA} {Time of Arrival}
\newacronym{TDOA}{TDOA} {Time Difference of Arrival}
\newacronym{RTOA}{RTOA} {Round-trip Time of Arrival}
\newacronym{WSN}{WSN} {Wireless Sensor Network}
\newacronym{WLAN}{WLAN} {Wireless Local Area Network}
\newacronym{RSS}{RSS} {Received Signal Strength}
\newacronym{RSSI}{RSSI} {Received Signal Strength Indication}
\newacronym{FSL}{FSL} {Free Space Loss}
\newacronym{BLE}{BLE} {Bluetooth Low Energy}
\newacronym{PE}{PE} {Potential Energy}
\newacronym{KE}{KE} {Kinetic Energy}
\newacronym{GRPR}{GRPR} {Golden Receiver Power Range}
\newacronym{ISM}{ISM} {Industrial, Scientific and Medical}
\newacronym{AP}{AP}   {Access Point}
\newacronym{MAC}{MAC} {Media Access Control}
\newacronym{IoT}{IoT} {Internet of Things}
\newacronym{LD}{LD}   {Log-Distance}
\newacronym{LQI}{LQI} {Link Quality Indicator}
\newacronym{SQC}{SQC} {Sum Quadratic Constraint}
\newacronym{RANSAC}{RANSAC}{RANdom SAmpling and Consensus}
\newacronym{RGB}{RGB} {Red-Green-Blue}
\newacronym{LED}{LED} {Light-Emitting Diode}
\newacronym{LoG}{LoG} {Laplacian of Gaussian}
\newacronym{SIFT}{SIFT} {Scale-Invariant Feature Transform}
\newacronym{SURF}{SURF} {Speeded-Up Robust Feature}
\newacronym{OF}{OF}   {Optical Flow}
\newacronym{FAST}{FAST} {Features from Accelerated Segment Test}
\newacronym{CenSurE}{CenSurE}{Center Surround Extremas for Realtime Feature Detection and Matching}
\newacronym{CC}{CC}   {Collision Cone}
\newacronym{VO}{VO}   {Velocity Obstacle}
\newacronym{RVO}{RVO} {Reciprocal Velocity Obstacle}
\newacronym{HRVO}{HRVO} {Hybrid Reciprocal Velocity Obstacle}
\newacronym{ORCA}{ORCA} {Optimal Reciprocal Collision Avoidance}
\newacronym{HL}{HL}   {Human-Like}
\newacronym{CALU}{CALU} {Collision Avoidance under Localization Uncertainty}
\newacronym{COCALU}{COCALU}{Convex Outline Collision Avoidance under Localization Uncertainty}
\newacronym{ROS}{ROS} {Robotics Operating System}
\newacronym{SIDPAC}{SIDPAC}{System Identification Programs for Aircraft}
\newacronym{STDMA}{STDMA}{Self-Organized Time Division Multiple Access}
\newacronym{UAV}{UAV} {Unmanned Air Vehicle}
\newacronym{MAV}{MAV} {Micro Air Vehicle}
\newacronym{LOFAR}{LOFAR}{LOw Frequency ARray}
\newacronym{OLFAR}{OLFAR}{Orbiting Low Frequency Array for Radio astronomy}
\newacronym{ANTS}{ANTS} {Autonomous Nano-Technology Swarm}
\newacronym{EDSN}{EDSN} {Edison Demonstration of Small-sat Networks}
\newacronym{COTS}{COTS} {Commercial Off-The-Shelf}
\newacronym{SI}{SI}   {Swarm Intelligence}
\newacronym{FP}{FP}   {Fokker-Planck}
\newacronym{IBDP}{IBDP} {Integral Birth-Death Process}
\newacronym{SPP}{SPP} {Self-Propelled Particles}
\newacronym{SSN}{SSN} {Swarm-Signaling Network}
\newacronym{TNI}{TNI} {Topological Neighborhood of Interaction}
\newacronym{ACO}{ACO} {Ant Colony Optimization}
\newacronym{PSO}{PSO} {Particle Swarm Optimization}
\newacronym{NASA}{NASA} {National Aeronautics and Space Administration}
\newacronym{ESA}{ESA} {European Space Agency}
\newacronym{JAXA}{JAXA} {Japan Aerospace Exploration Agency}
\newacronym{LEO}{LEO} {Low Earth Orbit}
\newacronym{GEO}{GEO} {Geosynchronous Earth Orbit}
\newacronym{AU}{AU}   {Astronomical Unit}
\newacronym{mDOT}{mDOT}  {Miniaturized Distributed Occulter/Telescope}
\newacronym{ARGoS}{ARGoS}{Autonomous Robots Go Swarming}

\newtheorem{theorem}{Theorem}
\newtheorem{lemma}{Lemma}
\newtheorem{proposition}{Proposition}
\newtheorem{definition}{Definition}
\usepackage{authblk}
\usepackage{blindtext}

\begin{document}

\title{Provable Emergent Pattern Formation by a Swarm of Anonymous, Homogeneous, Non-Communicating, Reactive Robots with Limited Relative Sensing and no Global Knowledge or Positioning}


\author{M. Coppola}
\author{J. Guo}
\author{E.K.A. Gill}
\author{G.C.H.E. de Croon}
\affil{Delft University of Technology, Faculty of Aerospace Engineering, \\Kluyverweg 1, 2629 HS Delft, The Netherlands}

\date{}

\maketitle



\begin{abstract}
  In this work, we explore emergent behaviors by swarms of anonymous, homogeneous, non-communicating, reactive robots that do not know their global position and have limited relative sensing.
  We introduce a novel method that enables such severely limited robots to autonomously arrange in a desired pattern and maintain it.
  The method includes an automatic proof procedure to check whether a given pattern will be achieved by the swarm from any initial configuration.
  An attractive feature of this proof procedure is that it is local in nature, avoiding as much as possible the computational explosion that can be expected with increasing robots, states, and action possibilities.
  Our approach is based on extracting the local states that constitute a global goal (in this case, a pattern).
  We then formally show that these local states can only coexist when the global desired pattern is achieved and that, until this occurs, there is always a sequence of actions that will lead from the current pattern to the desired pattern.
  Furthermore, we show that the agents will never perform actions that could a) lead to intra-swarm collisions or b) cause the swarm to separate.
  After an analysis of the performance of pattern formation in the discrete domain, we also test the system in continuous time and space simulations and reproduce the results using asynchronous agents operating in unbounded space.
  The agents successfully form the desired patterns while avoiding collisions and separation.
\end{abstract}

\section{Introduction}
\label{sec:introduction}
  Swarm intelligence enables several robots to collaborate, in a distributed fashion, towards achieving a common goal \citep{sahin2008swarm}.
  Each robot (or \emph{agent}) acts independently based on its local perception of the environment.
  The goal is achieved via the combined (inter-)actions of all agents \citep{navarro2012introduction}.
  The challenge is to develop controllers such that the swarm, despite the limited capabilities of each of its individuals, will 
  achieve the global goal (liveness property) and
  avoid undesired situations (safety property) 
  \citep{winfield2005towards}. \\

  In this paper, we focus on achieving these properties for the problem of pattern formation by a swarm of fully homogeneous, anonymous, and reactive robots.
  Homogeneous means that all robots are identical.
  Anonymous means that the robots do not know each other's identities (they cannot tell who is who).
  Reactive means that the robots react only based on their current perception of the world.
  Furthermore, we impose that each robot can only sense the environment in a narrow range around itself, and thus has only partial observability of the global state of the swarm.
  The agents also cannot communicate between each other and do not have any global positioning information.
  Finally, they hold no knowledge on the number of agents in the swarm or what the global goal actually is.
  Despite these extremely limited agents, we present a method that enables them to arrange into a desired pattern. \\

  The contribution of this article to the body of knowledge is a method to define local agent behaviors such that a highly limited swarm of agents can achieve a global pattern, with a proof procedure to check that it will achieve this from any initial configuration.
  The method is based on defining a set of desired local states 
  (as observable by an agent, directly based on its sensors) 
  that can coexist \emph{if and only if} the desired pattern is achieved.
  When an agent is in any other state, it will try to move about its neighbors in search of a desired state.
  With this approach, it is possible to guarantee that the swarm will reshuffle from any initial configuration into the desired pattern.
  Our proof of convergence is of a \emph{local} nature;
  it analyzes the role of agents within the swarm and the actions that they can take.
  The local proof is more restrictive than a global proof, but it avoids the computational explosion of a global analysis and does not make any statistical assumptions about the state of the agents.
  With the proposed method, one can show that the emergent behavior of a swarm is predictable, provided that its constituent local states and actions are properly defined.
  Additionally, the behavior can also guarantee that the agents remain aggregated and never collide (and thus satisfy the safety property).
  We validate the approach in a discrete environment and then export and test it in a simulation environment with continuous time, continuous space, and fully asynchronous agents. \\

  This paper is organized as follows.
  In \secref{sec:relatedwork}, we review prior approaches to distributed robot control with a focus on pattern formation, and we explain the context of this approach in the field.
  Following this, we define the problem and its key aspects in \secref{sec:problemdefinition}.
  The methodology, which includes the local proof, is then detailed in \secref{sec:frameworkdefinition}.
  The local proof requires that the desired pattern is the only pattern that can emerge.
  Our implementation to verify this is presented in \secref{sec:implementation}.
  The performance of the system is then tested for a set of representative patterns, first in an idealized discrete world (\secref{sec:gridsimulations}) and then in a continuous world (\secref{sec:swarmulatorsimulations}).
  The results and insights gathered are further discussed in \secref{sec:discussion}.
  Finally, \secref{sec:conclusion} provides concluding remarks and future research that can follow up on methodology and results in this paper.

\section{Related work and research context}
\label{sec:relatedwork}
  Distributed pattern formation and spatial organization is a branch of swarm robotics with applications in 
  aerial robots \citep{saska2016swarm, achtelik2012sfly},
  underwater robots \citep{joordens2010consensus},
  satellites \citep{engelen2011systems, verhoeven2011ontheorigin},
  planetary rovers \citep{kisdi2011future},
  and entertainment \citep{alonso2014collaborative}.\\

  One approach is to use a centralized omniscient controller that plans the path of every agent \citep{alonso2014collaborative}.
  This is efficient but requires external infrastructure, namely: 
  \begin{inparaenum}[1)]
    \item a global localization method and
    \item an external computer capable of communicating with the agents.
  \end{inparaenum}
  Distributed approaches aim to achieve the same result without the central controller.
  If a global localization method is still available, the agents can use their global position as a guide towards target locations \citep{hou2009multiplicative, morgan2015swarm, dada2013novel, gold2000utility}.
  However, for a swarm to be independent of external infrastructure, there is a need for algorithms that solely require on-board relative sensors. \\

  In a best case scenario, each robot can directly sense every other robot, in which case the swarm is \emph{fully connected} \citep{bouffanais2016network}.
  Several works operate under this assumption with valuable results.
  \cite{gazi2004stability} proved that a fully connected swarm can reach a stable formation for certain classes of attraction and repulsion functions.
  Such attraction and repulsion functions can also be manipulated to obtain certain patterns \citep{izzo2005equilibrium}.
  Alternatively, \cite{suzuki1999distributed}, \cite{flocchini2008arbitrary}, \cite{fujinaga2015pattern}, and \cite{guzel2017adaptive} devised explicit algorithms for arbitrary pattern formation.
  \cite{izzo2014evolutionary} and \cite{scheper2016abstraction} investigated the use of neural networks for asymmetrical pattern formation.
  \cite{yamauchi2014randomized} developed an algorithm that can replicate patterns following a leader election.
  \cite{pereira2008adaptive} and \cite{demarina2016distributed} used formation control algorithms.\\

  Unfortunately, the assumption that a swarm is fully connected does not hold for all applications.
  For instance, if robots sense each other with on-board cameras, they might be unable to see behind a nearby agent or beyond a certain distance.
  In this case, the swarm is \emph{connected} (agents can sense each other) but not fully connected.
  \emph{Sensing topology} is the graph of how agents in a swarm sense each other.
  \cite{tanner2004controllability} showed how certain sensing topologies can endow one agent with (indirect) control over all agents, and used this to manipulate the swarm to form patterns, under the assumption that the sensing topology is fixed.
  In swarms, however, the topology is likely not fixed, but changing depending on how the agents are positioned at any given time.
  It is then impossible for any agent to take advantage (or know about) its control centrality. \\

  Defining a functional hierarchy in the swarm can help to artificially endow certain agents with control centrality (regardless of the sensing topology), or to make agents act as ``seeds'' so that the other agents can use them as reference points.
  \cite{rubenstein2014programmable} notably used this approach to create patterns using a swarm of one thousand robots.
  Four seed agents acted as a static reference point for all other agents.
  Other methods that use seed/leader agents can be found in the works of 
  \cite{khaledyan2017formation}, 
  \cite{cicerone2016asynchronous}, 
  \cite{hasan2018circle}, and 
  \cite{wang2017pattern}.
  Furthermore, 
  \cite{derakhshandeh2016universal}, 
  \cite{diluna2017shape}, and 
  \cite{yamauchi2014randomized} explored autonomous leader election algorithms to avoid manually defining leaders.
  However, using leader/seed agents means that the other agents need to identify them as such, which cannot happen when the agents are anonymous and/or the conditions are such that they cannot all communicate exhaustively in order to agree on a leader.\\

  With a focus on self-assembly applications, \cite{klavins2002automatic} proposed a strategy for homogeneous and anonymous agents using graph structures.
  Robots could randomly move in an environment and latch together upon encounter.
  Based on a set of instructions, they could stick together if the latching matched a sub-element of the structure, or else detach.
  Over time, the agents would latch into the final assembly.
  Other similar works include the work of
  \cite{arbuckle2010selfassembly},
  \cite{arbuckle2012issues},
  \cite{klavins2007programmable},
  \cite{fox2015probabilistic},  and
  \cite{haghighat2017automatic}.
    These schemes assume that the agents can drift randomly in an enclosed area, and that by doing so they will eventually meet each other, at which point they will latch together.
  Once latched, the agents communicate in order to determine whether they should remain attached or whether they should detach and continue drifting until a new attachment is made. \\

  In this work we focus on a minimalist swarm and remove all assumptions to the maximum extent possible.
  We assume that each agent only has information about the relative location of its closest neighbors, and that they cannot afford to separate from the swarm as they operate in an unbounded environment.
  To our knowledge, there exists little research with such limited cases.
  \cite{krishnanand2005formations} explored using local information in order to align robots at specific bearings to each other, forming infinite grids or lines.
  \cite{flocchini2006gathering} explored the gathering problem, which is a special case of pattern formation where all robots have to gather towards the same location.
  \cite{yamauchi2013pattern} examined the formation power of very limited agents and how this related to the initial condition by comparing the symmetricities of the patterns. \\

  In this article, we present a general approach to design the local behaviors of robots such that they can form an arbitrary pattern, and we present a proof procedure to check that this can be achieved from any arbitrary initial configuration.
  This is shown to work for asynchronous agents and in an unbounded environment.
  The swarm can be proven to have the safety property (it never separates or experiences intra-swarm collisions) and the liveness property (the goal is eventually achieved).
  Based on their on-board sensors and actuators, the robots have a set of possible states that they can observe and actions that they can take.
  The state-space and the action-space of the agents are discretized, enabling a formal analysis of the system.
  This idea was introduced by \cite{winfield2005formal} and later explored by \cite{dixon2012towards} and \cite{gjondrekaj2012towards} using temporal logic paradigms, by explicitly defining the global states of the swarm in a formal environment.
  Using model checking techniques \citep{clarke1992model}, the idea behind these works is to verify emergent properties of the swarm by studying all its global states.
  However, as the size of the swarm grows, checking all global states leads to a computational explosion \citep{dixon2012towards}.
  This was tackled by \cite{konur2012analysing} with the use of macroscopic swarm models.
  Macroscopic model use statistics to model the behavior of the entire swarm, yet this bears the disadvantage that the analysis is only as good as the statistical approximation and the validity of its assumptions, which are not always applicable \citep{lerman2001macroscopic}.
  Therefore, in this work, we instead focus on a proof based on the \emph{local} states and actions of the agents.
  This has the advantage that it is independent of the size of the swarm and mitigates computational explosion, but it also does not make any statistical assumptions about the swarm.
  Using this approach, we can formally guarantee that the swarm remains aggregated, no collisions occur, and that there is always a sequence of actions that will lead it to form the global desired pattern.
  The global pattern is defined as a set of local states that build up the global state.
  Asynchronicity is tackled by allowing agents to move if and only if their neighbors are not moving.

\section{Problem Definition}
\label{sec:problemdefinition}
  In this work, the global emergent behavior of the swarm is forming a pattern.
  The goal of the swarm is to shuffle into the desired pattern and hold it despite none of the agents explicitly knowing that this is the global goal that they are trying to achieve, or being able to observe the global pattern.
  The swarm and its agents are under the following constraints:
    \renewcommand{\labelenumi}{C\theenumi:}
    \begin{enumerate*}
      \item The swarm is comprised of homogeneous agents (all agents are identical); \label{c:homogeneous}
      \item The agents are anonymous (they cannot know each other's ID); \label{c:anonymous}
      \item The agents are reactive (they only act based on their current state);
      \item The swarm is leaderless;
      \item The agents cannot communicate with each other;
      \item The agents make decisions locally (i.e., on-board);
      \item The agents do not know their global position;
      \item The agents exist in an unbounded space;
      \item Each agent can only sense the relative location of neighboring agents that are closer than a minimum range and within a field of view as allowed by their sensors.
      \label{c:rmax}
    \end{enumerate*}

    Furthermore, throughout this work, we make the following assumptions.
    \renewcommand{\labelenumi}{A\theenumi:}
    \begin{enumerate*}
      \item At the initial condition, the sensing topology of the swarm forms a connected graph; \label{a:connected}
      \item The agents all have knowledge of a common direction (e.g., North) and are programmed to act with respect to it;
      \footnote{
      On real robots, a common direction can be known using on-board sensors such as a magnetic sensor \citep{oh2015survey}.
      } \label{a:north}
      \item The agents exist and operate on a two-dimensional plane; \label{a:2d}
      \item When an agent senses the relative position of a neighbor, it can sense it with enough accuracy/frequency to establish if a neighbor is executing an action.
      \label{a:canimove}
    \end{enumerate*}
    \renewcommand{\labelenumi}{\theenumi.}
    \emph{Note}: Assumptions A\ref{a:north} and A\ref{a:2d} are not necessary.
    However, as will be shown by our framework, without Assumptions A\ref{a:north} the patterns that a swarm can generate become intrinsically limited due to the fact that the agents are unable to differentiate between certain states (this is further discussed in \secref{sec:discussion_thenorthdependency}).
    Assumption A\ref{a:2d} simplifies the analysis performed throughout the paper, but we expect our methods to also be extendable to three dimensions.
    Assumption A\ref{a:canimove} could also be challenged, but it will be shown that it is an important property of the robots if safe behavior is expected.
    The only assumption that cannot be removed is Assumption A\ref{a:connected}.
    This is because if a swarm does not start in a connected state (but, for instance, separated into two disconnected groups that cannot sense each other) then it cannot ever be expected for the groups to find each other in an unbounded space.

\section{Method Description}
\label{sec:frameworkdefinition}
  Each agent in a swarm can measure the relative location of its neighbors \textemdash{ }this forms the \emph{local state} of the agent.
  The principal idea is to extract the set of local states that the agents are in when the global pattern is achieved.
  These are the \emph{desired local states} of the agents.
  Being (or not being) in a desired local state can tell an agent something about the state of the whole swarm, similarly to how a puzzle piece tells something about a puzzle.
  As a simple example, if a swarm of 1000 robots had to form a line, and one agent sensed that it was surrounded by agents at all sides, then it could infer that the line has not yet been achieved (even if it is not able to sense all other agents).
  It could then move to amend the situation.
  Otherwise, it would remain where it is because from its perspective the global goal has been achieved.
  With this idea, we will show that just by informing the agents of a set of desired local states, we can cause them to reshuffle into the global pattern.\\

  To ensure that the desired pattern is the only emergent possibility, we need to check how the desired local states can coexist and see if the desired pattern is the only solution.
  If this is not the case, then we know that the agents are under-informed \textemdash{ }their sensors are insufficient to guarantee the pattern.
  This would require upgrading the sensors to, for instance, sense at a further range.
  If we know that the final pattern is unique, we can then analyze the local behavior of the agents to determine whether the pattern will be achieved from any initial condition.
  To conduct our analysis, we formally describe the sensory perception and the action capabilities of the agents.
  In doing so, we create a discrete description of what the robot can sense about its environment (local state space), and how it can move in this environment (local action space).

  \subsection{Local Sensor Layout and State Space Definition}
  \label{sec:localstatespace}
    With a sensor, a robot is expected to be capable of measuring the relative location of its neighbors. 
    In order to set up a formal framework, we formalize the sensor readings according a robot's sensor and its interplay with the expected inter-robot equilibrium distances (which may result, for instance, from attraction and repulsion forces).
    With this, we define the \textbf{sensor layout} of the robots.

    \begin{definition}
    	\label{definition:sensorpattern}
    	The \textbf{sensor layout} of a robot $\mathcal{R}$ is a boolean array $L$ of length $d_s$, i.e., $L_\mathcal{R} = \begin{bmatrix} l_1 & ... & l_{d_s} \end{bmatrix}$ arranged in space about the robot's frame of reference.
      The array specifies, from the perspective of robot $\mathcal{R}$, whether a neighbor is located at each relative position $l_i$ that is sensed 
    	($l_i=1$ if a neighbor is sensed at $l_i$, else $l_i=0$).
    \end{definition}

    \begin{definition}
    	\label{definition:link}
    	A \textbf{link} is an element $l_i=1$ in the sensory layout $L$, which indicates that a neighbor is sensed (``linked'') at that position.
    \end{definition}

    \begin{definition}
      \label{definition:localstate}
      A \textbf{local state} $s$ is a realization of a sensor layout $L$ based on whether a neighbor is sensed or not at each link $l_i$ in $L$.
    \end{definition}

    \figref{fig:state_space_design_options} shows a swarm of robots at equilibrium distances to each other and how this relates to different sensor layouts and the state of the robot. 
    Cases 1 and 2 feature omni-directional sensors which can sense at different ranges.
    Case 3 features a directional sensor (e.g., a camera).
    In Example Case 1, the local state of the agent is $s = \begin{bmatrix} l_1 & l_2 & l_3 & l_4 \end{bmatrix} = \begin{bmatrix} 0 & 0 & 1 & 0 \end{bmatrix}$, as also visually depicted in the figure.
    We then define the local state space $\mathcal{S}$ as the set of all possible local states that an agent can be in.
    It follows that $\mathcal{S}$ consists of all possible realizations of $L$, such that $|\mathcal{S}|=2^{d_s}$.
    In this paper, it is assumed that the swarm begins in a connected state (Assumption \ref{a:connected}) and we also show that the swarm never disconnects; 
    therefore, we eliminate the null state $s = \mathbf{0}^\mathsf{T}$ from $\mathcal{S}$, such that $|\mathcal{S}|=2^{d_s}-1$.
    Furthermore, for representation purposes, we will focus on the case where robots have an omni-directional sensor as seen in Case 1 and 2, although the methods in the paper can be extended to other sensor layouts. \\

    \begin{figure}[ht!]
      \centering
        \includegraphics[width=\textwidth]{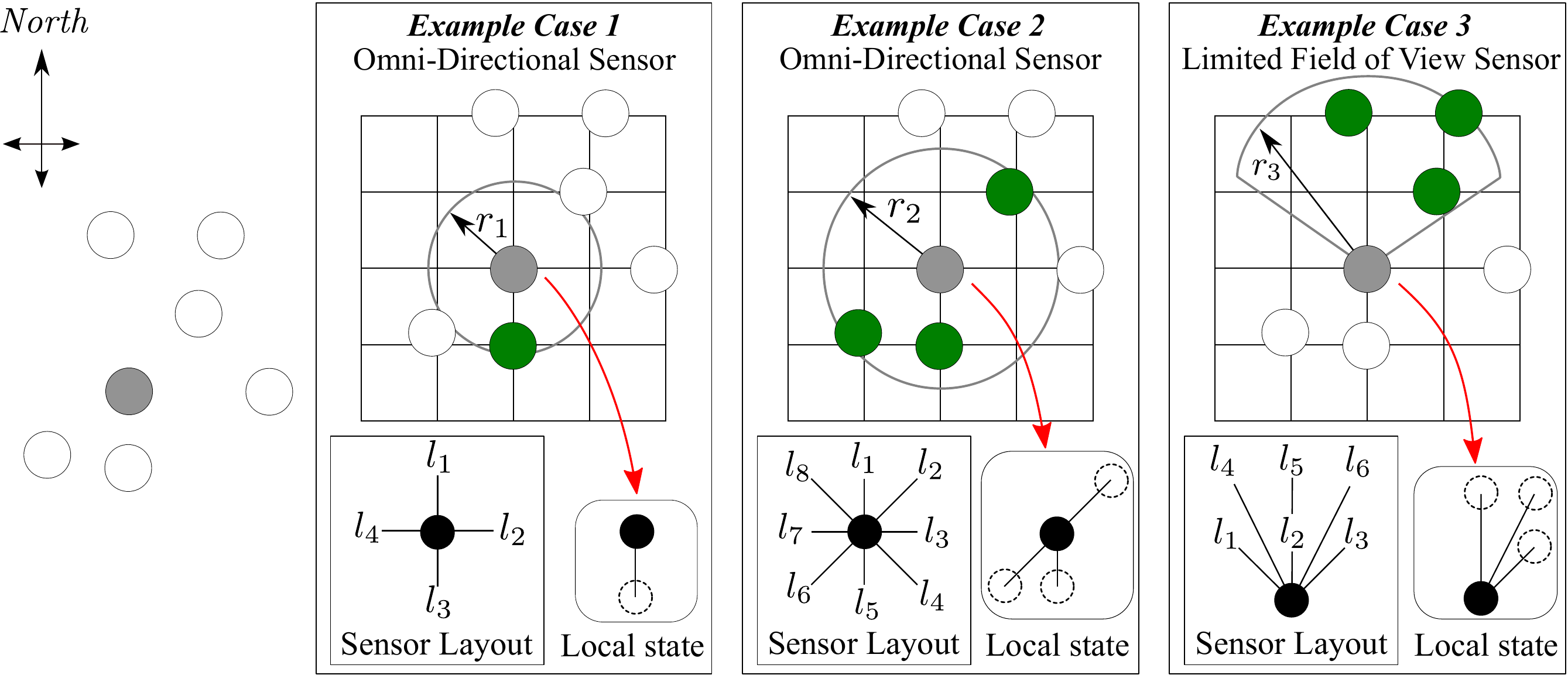}
      \caption{Examples of three different sensor layouts and the resulting local state of an agent in the swarm. Notice that the agents exist in continuous space, but their presence is discretized to the closest grid point.}
      \label{fig:state_space_design_options}
    \end{figure}
  	
  	Local states of neighboring robots must be able to coexist.
  	If robot $\mathcal{R}_1$ sees robot $\mathcal{R}_2$, then it follows that  robot $\mathcal{R}_2$ should also see robot $\mathcal{R}_1$ (if the two are both in each other's field of view).
  	Then, if robot $\mathcal{R}_1$ sees a third robot $\mathcal{R}_3$ at a position where $\mathcal{R}_2$ should also be able to see robot $\mathcal{R}_3$, then it follows that robot $\mathcal{R}_2$ should also see the third robot $\mathcal{R}_3$, and so on.
  	If this is respected, then we say that the local states of robots \textbf{match}. 
  	Examples of local states that do not match and states that match are visualized in \figref{fig:matching_states}.
  	\begin{definition}
      \label{def:match}
  		Two (or more) local states \textbf{match} when they do not have conflicting information about the relative location of their neighbors and each other.
  	\end{definition}

    \begin{figure}[ht!]
      \centering
        \includegraphics[width=0.5\textwidth]{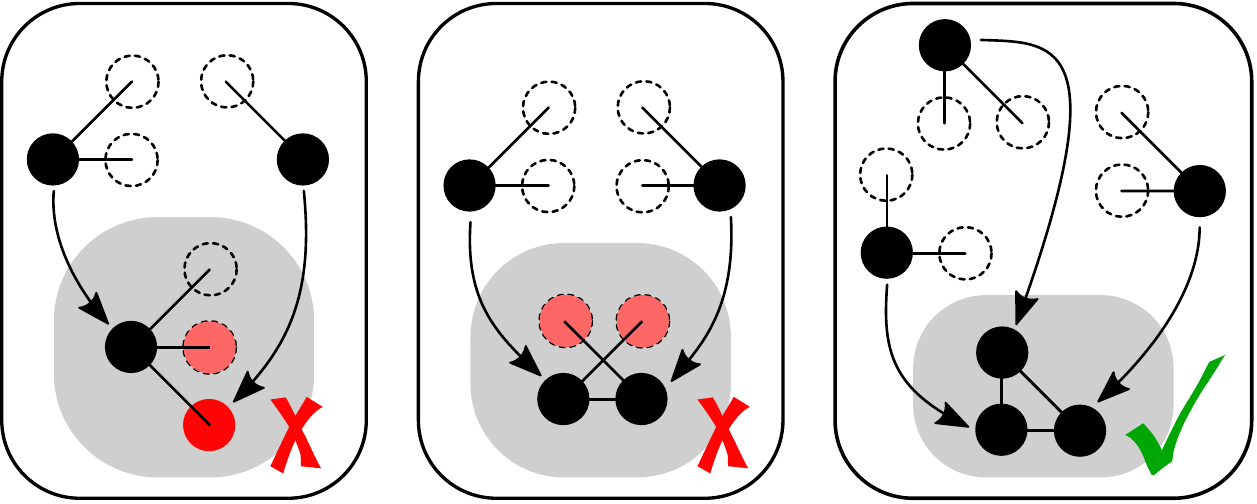}
      \caption{Examples of local states that do not match (left, center) and states that match (right). The sensor layout from Case 2 in \figref{fig:state_space_design_options} is used.}
      \label{fig:matching_states}
    \end{figure}
  	
  \subsection{Action Space Definition}
  \label{sec:localactionspace}
    An action is a motion that the agent can perform in space.
    Similarly as to the state space, we discretize actions with respect to an egocentric frame of reference.
    Let $\mathcal{A}$ be the action space, which is dependent on the actuators available and the degrees of freedom of the robots.
    As illustrated in \figref{fig:action_space_design_options}, a robot that can move in all directions, such as quad-rotors or certain ground robots, would be described with an omni-directional action space. 
    A more limited robot could be described with a constrained action space.

    \begin{figure}[ht!]
      \centering
        \includegraphics[width=0.6\textwidth]{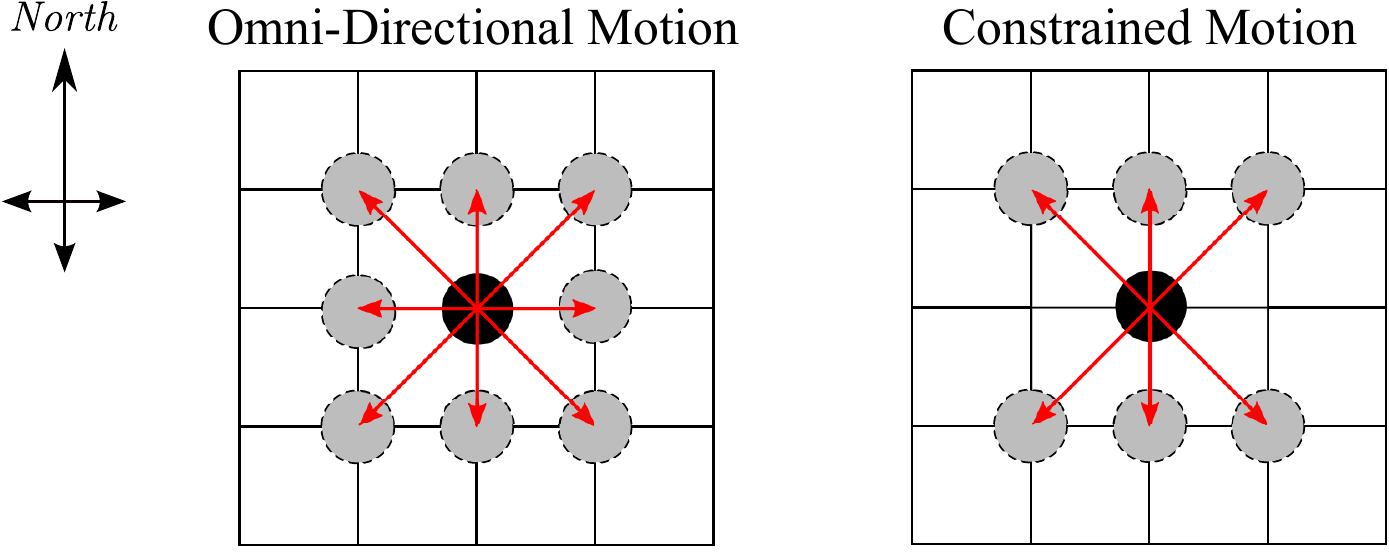}
      \caption{Examples of possible action spaces}
      \label{fig:action_space_design_options}
    \end{figure}

  \subsection{Determining the Desired Local States Needed to Achieve a Pattern}
  \label{sec:desiredstates}
    Based on a desired global pattern and a given sensor layout, we can extract the local states that the agents are in when a desired pattern is composed.
    These states are referred to as the desired states, and are grouped under the set $\mathcal{S}_{des}$, where $\mathcal{S}_{des}\subseteq\mathcal{S}$.
    This process is analogous to extracting the puzzle pieces that create a puzzle.  
    In the general case, the size of set $\mathcal{S}_{des}$ does not need to be equal to the number of agents in the swarm $N$.
    Any number of agents can be assigned to any set $\mathcal{S}_{des}$.
    The patterns that can be formed stem from any possible matching of $N$ of the states in $\mathcal{S}_{des}$, with repetition.
    For any set $\mathcal{S}_{des}$ and a swarm consisting of a fixed number of agents $N$, we thus have one of four possible outcomes:
    \begin{enumerate*}
      \item \emph{No pattern is possible}: $N$ instances of the states in $\mathcal{S}_{des}$ do not match in any way.
      \item \emph{Only undesired patterns are possible}:
        it is impossible to settle in the desired pattern but other patterns are possible.
      \item \emph{Desired pattern is possible}:
        it is possible to settle in the desired pattern but other patterns are also possible.
      \item \emph{Desired pattern is possible and unique}:
        it is only possible to settle in the desired pattern.
    \end{enumerate*}
    \figref{fig:examples_outcomes} visually shows the four outcomes for different sets of $\mathcal{S}_{des}$, a swarm of 4 agents, and a specific sensor layout.
    Each of the 4 agents in the swarm can be any of the states in $\mathcal{S}_{des}$.
    In this example we deal with a small swarm, making it is possible to visually extract possible patterns.
    However, as the size of the swarm and the size of $\mathcal{S}_{des}$ grows, there is a need for an automatic checker.
    Our implementation of this is detailed in \secref{sec:implementation}.

    \begin{figure}[ht!]
      \centering
        \includegraphics[width=\textwidth]{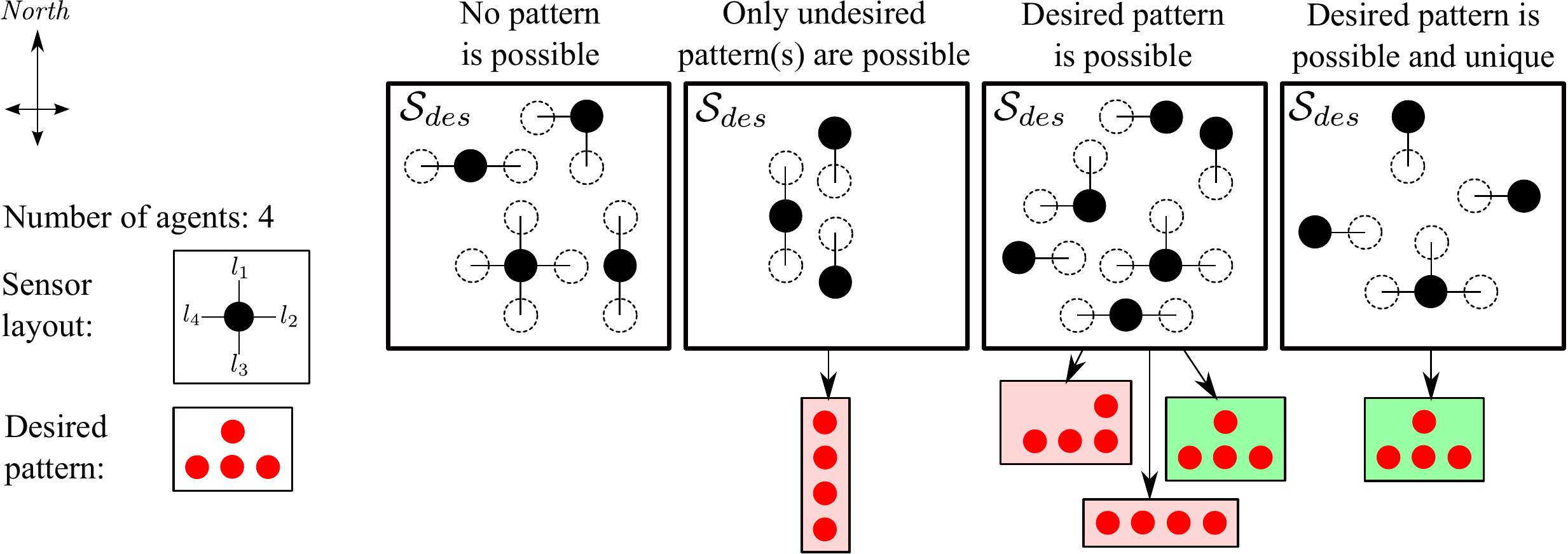}
        \caption{Examples of outcomes for different sets $\mathcal{S}_{des}$ by a swarm of 4 agents using the sensor layout from Case 1 in \figref{fig:state_space_design_options}}
        \label{fig:examples_outcomes}
    \end{figure}

  \subsection{Defining a Safe State-Action Relation}
  \label{sec:stateaction_space}
    In this section, we wish to develop a state-action relation such that a connected swarm remains \textbf{safe} at all times, as defined in Definition \ref{def:safe}.
    \renewcommand{\labelenumi}{\theenumi)}
    \begin{definition}
    	\label{def:safe}
    	A connected swarm remains \textbf{safe} if neither of the following events occur:
    	\begin{inparaenum}
    	\item a collision between two or more agents, 
    	\item the swarm disconnects.
    	\end{inparaenum}
    \end{definition}

    Our swarm consists of several asynchronous agents that can choose to take actions at any point in time.
    Safety can be guaranteed when agents do not simultaneously perform conflicting actions.
    To formalize this, we bring forward Proposition \ref{proposition:oneatatime}.

    \begin{proposition}
      \label{proposition:oneatatime}
      If the swarm never features more than one agent moving at the same time, then the swarm can remain safe.
    \end{proposition}
    \begin{proof}
      Consider a connected swarm organized into an arbitrary pattern $P$.
      At a given time $t=t_1$, agent $i$ decides to take an action based on action space $\mathcal{A}$.
      This action should last until $t=t_{2}$.
      However, at time $t_1<t<t_2$, an unsafe event takes place.
	  It follows that the event must have been the fault of agent $i$, because it was the only agent that moved.
	  Therefore, if agent $i$ could select only from safe actions, this would be sufficient to guarantee that the swarm is safe at time $t=t_2$.
    \end{proof}

    Given the constraints of our robots, Proposition \ref{proposition:oneatatime} can be used in a formal analysis of the system, but it cannot be used in a real system.
    This explains the importance for Assumption A\ref{a:canimove} in \secref{sec:problemdefinition}:
    an agent must know whether its neighbors are executing an action in order to avoid causing conflicts.
    If the agents do not move whenever any of their neighbors are moving (according to a first come first served basis), then the swarm approaches the formal requirement of Proposition \ref{proposition:oneatatime}, and the swarm can remain safe provided that the moving agent executes safe actions.
    \footnote{This will be explored in \secref{sec:swarmulatorsimulations}, where the system is tested with fully asynchronous agents in a continuous time and continuous space setting.}
    To define which actions are safe, we bring forward Propositions \ref{proposition:motion} and \ref{proposition:separation}.

    \begin{proposition}
      \label{proposition:motion}
      If an agent is the only agent moving in the entire swarm, and this agent only selects actions in directions that can be sensed by its on-board sensors, then it can be guaranteed that collisions will not occur.
    \end{proposition}
    \begin{proof}
      Consider an agent $i$ in a swarm.
    	Following Proposition \ref{proposition:oneatatime}, we know that the agent will be the only agent to move.
      The agent moves in the environment according to the action space $\mathcal{A}$.
      If all actions in $\mathcal{A}$ lead to a location that is already sensed, then agent $i$ can establish whether the action will cause a collision, and it can choose against performing these actions.
    \end{proof}

    \begin{proposition}
      \label{proposition:separation}
      If an agent is the only agent moving in the entire swarm, and the agent only takes actions such that, at its new location, all its prior neighbors and itself remain connected, then the swarm will remain connected.
    \end{proposition}
    \begin{proof}
      Consider a connected swarm of $N$ agents.
      The graph of the swarm is connected if any node (agent) $i$ features a path to any other node (agent) $j$.
      Consider the case where agent $i$ takes an action.
      If, following the action, agent $i$ is still connected to all its original neighbors, then the connectivity of the graph was not affected.
      If agent $i$ only selects actions where, at its final position, this principle is respected, then it will be able to move while guaranteeing that the swarm remains connected.
    \end{proof}

    Using Propositions \ref{proposition:motion} and \ref{proposition:separation} we can extract a state-action space where safe actions are guaranteed.
    Let $s$ be the state of an agent, and $\mathcal{S}$ be the set of all local states that an agent can be in.
    Agents with state $s\in\mathcal{S}_{des}$ do not move, as they wish to remain in their current state.
    Ideally, all other agents would move until they all achieve a state $s\in\mathcal{S}_{des}$, at which point the desired pattern is formed.
    The full state-action map is thus given by: 
    $\mathcal{Q}=\mathcal{S}_{\lnot des} \times \mathcal{A}$, where $\mathcal{S}_{\lnot des}= \mathcal{S} \setminus \mathcal{S}_{des}$.
    We then scan through $\mathcal{Q}$ to identify all state-action pairs that:
    \begin{enumerate}[a)]
      \item \textbf{are in the direction of a neighbor.}\\
        These state-action pairs will lead to collisions.
        They form the set $\mathcal{Q}_{collision}$.

      \item \textbf{feature an action in a direction that is not sensed.}\\
        Following Proposition \ref{proposition:motion}, these potentially lead to collisions.
        They form the set $\mathcal{Q}_{blind}$.

      \item \textbf{may cause the graph to disconnect.}\\
        Following Proposition \ref{proposition:separation}, these actions will break the local connectivity, with potentially a global impact.
        They form the set $\mathcal{Q}_{separation}$.
    \end{enumerate}

    We then define:
    \begin{equation}
      \mathcal{Q}_{safe} = \mathcal{Q} \setminus (\mathcal{Q}_{collision} \cup \mathcal{Q}_{blind} \cup \mathcal{Q}_{separation})
      \label{eq:q_reduced}
    \end{equation}
    $\mathcal{Q}_{safe}$ is a state-action mapping that only includes safe state-action relations.
    It may be that not all states from $\mathcal{S}_{\lnot des}$ are present in $\mathcal{Q}_{safe}$.
    An agent in such a state would be unable to select any safe action \textemdash{ }it would be \emph{blocked}.
    All states in which an agent would be blocked are grouped under $\mathcal{S}_{blocked}$.
    Functionally speaking, the states in $\mathcal{S}_{blocked}$ and $\mathcal{S}_{des}$ are equivalent.
    In either case, the agent will not move.
    Based on this, we create a new set $\mathcal{S}_{static}$ and its complement $\mathcal{S}_{active}$.
    \begin{align} 
      \mathcal{S}_{static} &= \mathcal{S}_{des} \cup \mathcal{S}_{blocked} \\
      \mathcal{S}_{active} &= \mathcal{S} \setminus \mathcal{S}_{static}
    \end{align}
    It is thus $\mathcal{S}_{static}$, and not only $\mathcal{S}_{des}$, that  should be checked in order to assess whether the desired pattern is a unique solution.

  \subsection{Agent Behavior and Local Proof of Convergence}
  \label{sec:pursuitofhappiness}
    The behavior of an agent is presented as a \gls{FSM} in \figref{fig:statemachine}.
    An agent can be in one of two macro states:
    \renewcommand{\labelenumi}{}
    \begin{enumerate*}
      \item \textbf{Static:} The agent is in state 
      $s\in\mathcal{S}_{static}$ and is unable to move.
      \item \textbf{Active:} The agent is in state 
      $s\in\mathcal{S}_{active}$ 
      and can take an action based on $\mathcal{Q}_{safe}$
    \end{enumerate*}
    \renewcommand{\labelenumi}{\theenumi.}
    \begin{figure}[ht!]
      \centering
      \includegraphics[width=0.4\textwidth]{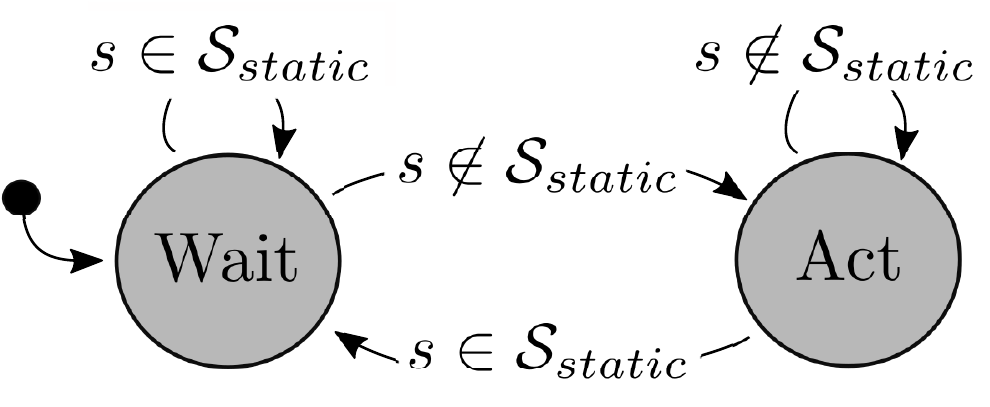}
      \caption{FSM of agent behavior}
      \label{fig:statemachine}
    \end{figure}

    We now provide the local proof, and the necessary conditions, for the desired pattern to be achieved from any initial configuration of the swarm, provided that the swarm initiates in a connected state (Assumption A\ref{a:connected}).
    In the following analysis, let $P_0$ be an arbitrary initial pattern formed by a swarm, and let $P_{des}$ be the desired final pattern.
    $N_{des}$ is the number of agents that are needed in order to compose $P_{des}$.
    Here, we assume that the swarm is always composed of $N_{des}$ agents.
    The only formation in which all agents are static is the desired formation $P_{des}$.
    \footnote{This property, assumed in this section, needs to be checked independently. Our implementation is detailed in \secref{sec:implementation}.}
    To reflect Proposition \ref{proposition:oneatatime}, we formally analyze how the swarm evolves by modeling actions in discrete time steps.
    At each time step $k$, an arbitrary agent with state $s\in\mathcal{S}_{active}$ executes an action based on $\mathcal{Q}_{safe}$.\\

    We begin by establishing that, for any pattern $P\neq P_{des}$, there is always one active agent, as per Lemma \ref{lemma:active_states_present}.
    \begin{lemma}
      For a swarm of $N_{des}$ agents, if $\mathcal{S}_{static}$ is such that the desired pattern $P_{des}$ is possible and unique, any arbitrary pattern $P\neq P_{des}$ will feature at least $1$ agent with a state $s\in\mathcal{S}_{active}$.
      \label{lemma:active_states_present}
    \end{lemma}
    \begin{proof}
      By definition:
      $\mathcal{S}_{static}\cap \mathcal{S}_{active}=\emptyset$ and 
      $\mathcal{S}_{static}\cup \mathcal{S}_{active}=\mathcal{S}$.
      For a swarm of $N_{des}$ agents that can be in states $s\in\mathcal{S}$, $N_{des}$ instances of states $s\in\mathcal{S}_{static}$ can only coexist into $P_{des}$, which is known to be the unique outcome.
      Therefore, it follows that any other pattern must feature at least one agent that is in a state $s\not\in\mathcal{S}_{static}$, meaning that it is in a state $s\in\mathcal{S}_{active}$.
    \end{proof}

    Following Lemma \ref{lemma:active_states_present}, we must determine whether the actions taken by the agents can lead to forming $P_{des}$ starting from any initial pattern $P_0$.
    Overall, when an agent transitions from state $s$ to a state $s'$, its transition can be of three types:
    \renewcommand{\labelenumi}{\emph{Transition Type} \theenumi:}
    \begin{enumerate*}
      \item By its own action, if $s\in\mathcal{S}_{active}$, via an action from $\mathcal{Q}_{safe}$ (when this happens, some neighbors might leave from view, while new neighbors might come into view).
      \item By the action of a neighbor (when this happens, the neighbor could also move out of view).
      \item By a new agent, previously outside of view, moving into view and becoming a new neighbor.
    \end{enumerate*}
    \renewcommand{\labelenumi}{\theenumi.}
    Let $G_{\mathcal{S}} = (V,E)$ be a directed graph, where $V$ is a set of vertices (or nodes) and $E$ is a set of edges.
    We will let each node of $G_{\mathcal{S}}$ represent each local state in $\mathcal{S}$, such that $V=\mathcal{S}$.
    The edges of the graph $E$ are all local state transitions that an agent can experience as a result of the three transition types above.
    We will define $E_1$ as the edges describing Transition Type 1, 
    $E_2$ as the edges describing Transition Type 2, and
    $E_3$ as the edges describing Transition Type 3.
    Based on this, let $G_{\mathcal{S}}^{x}$ denote the subgraph of $G_{\mathcal{S}}$ that only focuses on Transition Type $x$ (i.e., has edges $E_x$), such that 
    $\bigcup_{x=1}^3 G_{\mathcal{S}}^{x} = G_{\mathcal{S}}$.
    The graphs $G_S^1$, $G_S^2$, $G_S^3$, and $G_S$ are illustrated in \figref{fig:gs}.
    Using this, we present Lemma \ref{l:achievability}, which expresses the conditions needed for a pattern to be \textbf{achievable}, as defined by Definition \ref{definition:achievable}.
    Note that the condition in this Lemma does not imply that the pattern will be achieved from any initial configuration of the swarm (this comes later in the section), but it only establishes whether it is within the capabilities of the agents to achieve the local states required to make the pattern.

	  \begin{figure}[ht!]
	    \centering
	    \includegraphics[width=\textwidth]{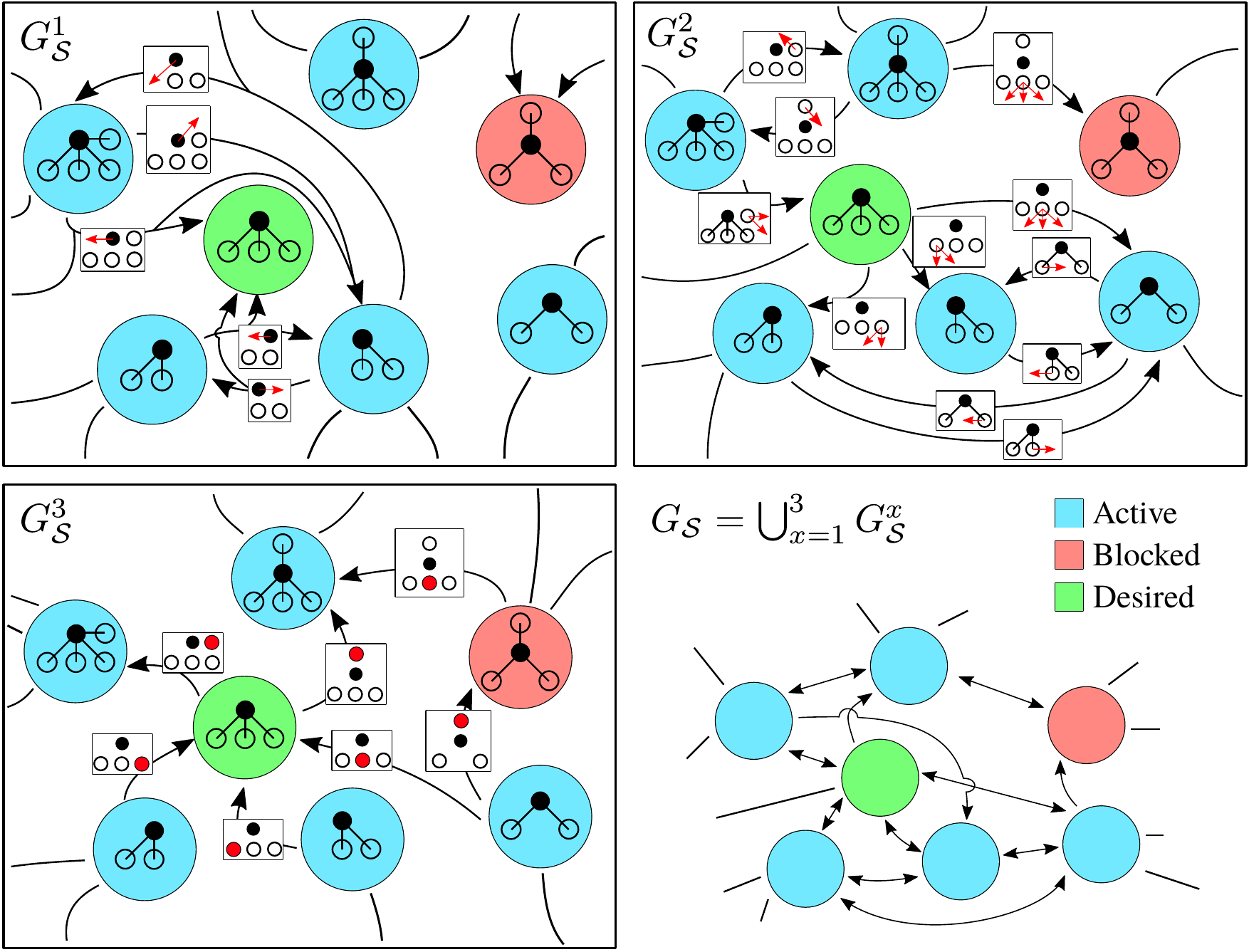}
	    \caption{
	    Depiction of how the local states of an agent can change as a result of movements in the swarm.
      Specifically, the figure shows a portion of graph $G_S$ and its subgraphs $G_S^1$, $G_S^2$, $G_S^3$.
  	  Graph $G_S^1$ corresponds to Transition Types 1 (the edges depict the agent taking an action, each action can lead to different outcomes depending on whether new agents comes into view as a result of the action).
  	  Graph $G_S^2$ corresponds to Transition Types 2 (the edges depict the neighbors of the agent taking an action).
  	  Graph $G_S^3$ corresponds to Transition Types 3 (the edges depict a new agent coming into view).
      Green nodes indicate a desired state, blue nodes indicate an active state, and red nodes indicate a blocked state.
      The agents have an omni-directional sensor layout as in 
      Case 2 from \figref{fig:state_space_design_options} 
      and omni-directional motion 
      as seen in \figref{fig:action_space_design_options}.
	    }
	    \label{fig:gs}
	  \end{figure}

    \begin{definition}
      \label{definition:achievable}
      A pattern $P_{des}$ is \textbf{achievable} if each of its local constituent states in 
      $\mathcal{S}_{des}$ can be reached starting from any local state in $\mathcal{S}$.
    \end{definition}

    \begin{lemma}
      \label{l:achievability}
      If the digraph
      $G_\mathcal{S}^1 \cup G_\mathcal{S}^2$ shows that each state in 
      $\mathcal{S}$ features a path to each state in $\mathcal{S}_{des}$, 
      then $P_{des}$ is achievable independently of the local states that compose $P_0$.
    \end{lemma}
    \begin{proof}
      $P_{des}$ is formed if and only if all agents have a state
      $s\in\mathcal{S}_{des}$, where
      $\mathcal{S}_{des}\subseteq\mathcal{S}$.
      Consider an arbitrary initial pattern $P_0$ for which the local states of the agents form an arbitrary set $\mathcal{S}_0$.
      Via Lemma \ref{lemma:active_states_present} we know that there is at least one agent in the swarm that is active for any pattern 
      $P_0\neq P_{des}$, 
      and in turn any set of states $S_0\neq S_{des}$.
      As the active agents move, they will experience transitions described by 
      $G_\mathcal{S}^1$, and their neighbors will experience transitions described by $G_\mathcal{S}^2$.
      By the unified graph $G_\mathcal{S}^1 \cup G_\mathcal{S}^2$ we describe the local transitions that take place to an agent as it moves and as its neighbors move.
      Consider a state $s\in\mathcal{S}_0$ that is incapable (either by its own actions or by the actions of its potential neighbors) to transition to a state in $\mathcal{S}_{des}$.
      It follows that having this state in $S_0$ may mean that a state in $S_{des}$ cannot be achieved, and in turn that $P_{des}$ cannot be realized.
      However, if it is possible for any state in $\mathcal{S}$ to experience local transitions such that it may reach any state $\mathcal{S}_{des}$, it follows that $P_{des}$ is achievable independently of the local states that compose $P_0$ (i.e, the set $\mathcal{S}_0$), because there is no state $s\in\mathcal{S}_0$ that is incapable of executing the necessary transitions that would lead it to be in a state $\mathcal{S}_{des}$.
      By ignoring the role of $G_\mathcal{S}^3$, we restrict the system such that:
        \begin{enumerate*}
          \item Any state $s$ that has too few links for a desired state will have to be active and move to a position where it is surrounded by enough agents.
          \item Any state $s\in\mathcal{S}_{blocked}$ can become active by the actions of a neighbor.
          \item The transitions that occur must occur because of changes in the local neighborhood.
        \end{enumerate*}
      This additional restriction ensures that the system can rely on the actions of an agent and/or its neighbors.
    \end{proof}

    Lemma \ref{l:achievability} decides to ignore the possible role of $G_\mathcal{S}^3$ in order to be more restrictive.
    This ensures that the transitions do not rely on agents coming into view.
    The conditions of Lemma \ref{l:achievability} guarantee that any initial state could potentially turn into a desired state, such that there are no restrictions on the local states that could compose $P_0$.
    However, this is not equivalent to saying that $P_{des}$ will always eventually be formed from any arbitrary $P_0$, which is the property that we wish to achieve. 
    To extract conditions that guarantee that $P_{des}$ can be achieved from any $P_{0}$, we look into the properties of a global graph $G_P$.
    The nodes of $G_P$ are all possible patterns that the swarm can generate (including $P_{des}$), and the edges are all possible transitions between patterns that can exist as a result of an action taken by an active agent.
    Therefore, if the properties of $G_P$ are such that there is a path from any node to the node for $P_{des}$, and that this path is free of \textbf{deadlocks} 
    (see Definition \ref{definition:deadlock})
    then we know that $P_{des}$ can be reached from any initial pattern $P_0$.

    \begin{definition}
      \label{definition:deadlock}
      A \textbf{deadlock} is a situation in which the swarm continuously transitions into the same sequence of patterns, e.g. $P_0\rightarrow P_1\rightarrow P_2\rightarrow P_0\rightarrow P_1\rightarrow P_2\rightarrow P_0...$, and cannot transition to any other pattern.
    \end{definition}

    At this point, a solution to determine that $P_{des}$ can be reached from any initial pattern $P_0$, without deadlocks, would be to compute $G_P$ and directly inspect whether the property is fulfilled, but this would come at the cost of a large state explosion \citep{dixon2012towards}.
    Therefore, we instead continue with a local proof and impose local conditions that guarantee that $G_P$ will have the desired properties.
    Inherently, as for Lemma \ref{l:achievability}
    (which ignored the possible role of $G_S^3$),
    this proof comes at the cost of some additional local level restrictions which may not be necessary for all global patterns.
    However, a local approach enables us to determine sufficient conditions while only using the local state information, and is thus independent of the size of the swarm.
    The approach focuses on the role of \textbf{simplicial} states (Definition \ref{definition:simplicial}) and their \textbf{cliques}  (Definition \ref{definition:clique}).
    \footnote{
    	These definitions are borrowed from, but not equivalent to, the typical definitions of \emph{simplicial node} and \emph{clique}.
	    In standard graph theory, a simplicial node is a node whose neighboring nodes are \emph{fully} connected, not just connected, and thus form only one clique, and not several \citep{vansteen2010graph}.
	  }

    \begin{definition}
      \label{definition:simplicial}
      A \textbf{simplicial} state is a state $s\in\mathcal{S}$ 
      for which, if that agent were to move away completely out of the neighborhood, its original neighbors would remain connected.
      Note that if all neighboring positions are occupied, then the agent cannot move away, so an agent in this state is not simplicial.
      The set of simplicial states is denoted $\mathcal{S}_{simplicial}$, where $\mathcal{S}_{blocked}\cap \mathcal{S}_{simplicial} = \emptyset$.
      The set of states that are not simplicial is denoted $\mathcal{S}_{\lnot simplicial}$.
    \end{definition}

    \begin{definition}
      \label{definition:clique}
	    A \textbf{clique} is a connected set of neighbors of an agent.
	    If the agent were removed, the members of each clique would remain connected amongst each other, but there would not be a connection between the different cliques.
	    It follows that simplicial agents only have one clique, while blocked agents have two or more cliques.
    \end{definition}

    Agents in a state $s\in\mathcal{S}_{simplicial}\cap\mathcal{S}_{active}$ can move away from their neighborhood \textemdash{ }this is an important property.
    Intuitively, a simplicial active agent can cause the swarm to reshuffle into different achievable patterns without deadlocks.
    Alternatively, when simplicial agents are not present, this may not be possible.
    As exemplified in \figref{fig:gp_b}: when a pattern with no simplicial agents is reached, the active (non-simplicial) agents cannot remove themselves from the neighborhood and solve the deadlock situation.
    However, because the pattern in \figref{fig:gp_a} has simplicial agents that can travel around all others, we can always begin from any initial condition and transition to $P_{des}$.
    The proof we present stems from this intuition.
    We first present a local condition to ensure that any node (pattern) in $G_P$ always eventually transitions to a node with at least
    one active simplicial agent (unless $P_{des}$ is reached).
    Then, we will prove that this property enables for a graph $G_P$ that always features a path to $P_{des}$ from any initial condition.\\
 
    \begin{figure}[ht]
	    \centering
      \begin{subfigure}[t]{0.34\textwidth}
        \centering
        \includegraphics[width=\textwidth]{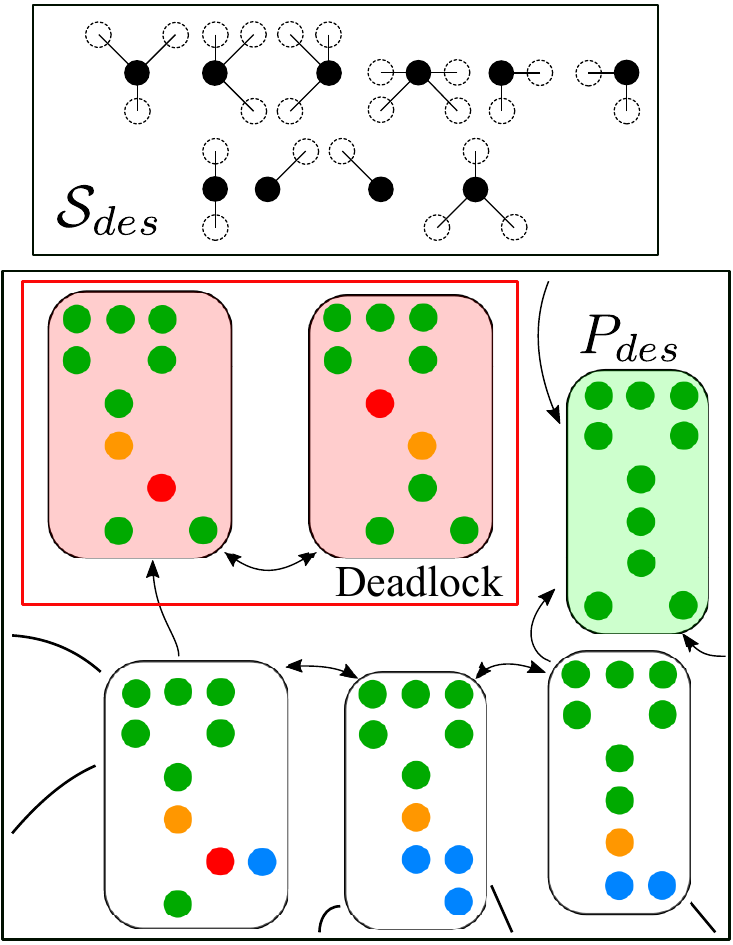}
        \caption{Example of a deadlock}
        \label{fig:gp_b}
      \end{subfigure}
      \hspace{2mm}
      \begin{subfigure}[t]{0.63\textwidth}
        \centering
        \includegraphics[width=\textwidth]{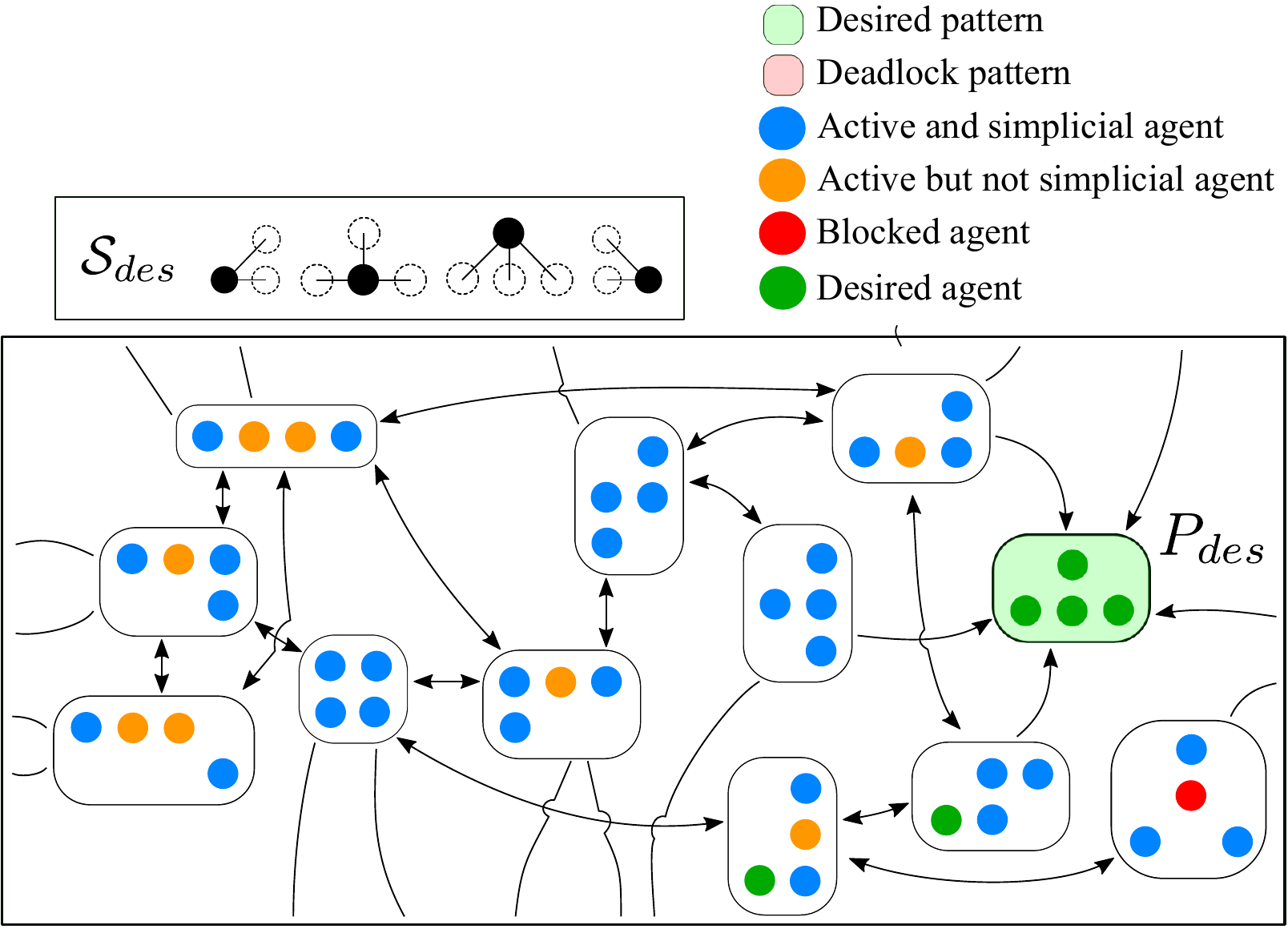}
        \caption{Pattern where a deadlock is not possible}
        \label{fig:gp_a}
      \end{subfigure}
	    \caption{
      Illustrations of how a swarm can transition between different patterns, based on movements of the agents that are in active states. 
	    More specifically, the figure shows a portion of $G_P$ for two possible desired patterns.
      Notice that the deadlock in (a) does not feature any simplicial active agents.
      The agents have an omni-directional sensor layout as in 
      Case 2 from \figref{fig:state_space_design_options} 
      and omni-directional motion 
      as seen in \figref{fig:action_space_design_options}.
	    }
	    \label{fig:gp}
  	\end{figure}

    Consider the graph $G_P^{AS} \subseteq G_P$  (where the superscript ``AS'' stands for ``Active and Simplicial'').
    The nodes of $G_P^{AS}$ are all nodes of $G_P$ which feature one of the following:
    \begin{enumerate*}
    	\item one or more states $s\in\mathcal{S}_{simplicial}\cap\mathcal{S}_{active}$. We group all these patterns in the set $P_{AS}$.
    	\item only states in $\mathcal{S}_{static}$ (this is the pattern $P_{des}$).
    \end{enumerate*}
    Therefore, the nodes of $G_P^{AS}$ are all patterns $P\in P_{AS}\cup P_{des}$.
    The edges of $G_P^{AS}$ are all edges that connect these nodes as in $G_P$.
    Due to the importance of active and simplicial agents in avoiding deadlocks, we wish to ensure that a pattern in $G_P^{AS}$ will be reached from any pattern in $G_P$.
    The conditions for this depend on the set $\mathcal{S}_{active}$.
    If $\mathcal{S}_{active}=\mathcal{S}_{simplicial}$, following Lemma \ref{lemma:active_states_present}, then $G_P^{AS} = G_P$.
		If $\mathcal{S}_{active}\cap\mathcal{S}_{\lnot simplicial}\neq\emptyset$, 
    then we must impose further restrictions, put forward by Lemma \ref{lemma:activeandsimplicialpresent}.
    In this Lemma we also make use of a graph $G_\mathcal{S}^{2r}\subseteq G_\mathcal{S}^{2}$, which only considers the transitions in $G_\mathcal{S}^2$ that do not feature an agent leaving the neighborhood, but only move about the agent.

    \begin{lemma}
      \label{lemma:activeandsimplicialpresent}
      If the following conditions are satisfied:
      \begin{enumerate*}
      	\item for all states $s\in\mathcal{S}_{blocked}\cap\mathcal{S}_{\lnot des}$, none of the cliques of each state can be formed uniquely by agents that are in a state
      	$s\in\mathcal{S}_{des}\cap\mathcal{S}_{simplicial}$,

      	\item $G_\mathcal{S}^{2r}$ shows that all static states with two neighbors can directly transition to an active state,
      \end{enumerate*}
      then the nodes in $G_P^{AS}$ will be reached from any other node in $G_P$.
    \end{lemma}
    \begin{proof}
    	A blocked agent $i$ with state $s_i\in\mathcal{S}_{blocked}$ always has multiple agents surrounding it, or else it would not be blocked.
    	The neighbors of agent $i$ either form two or more cliques, or they form one clique that fully surrounds the agent in all sensed directions.
    	In either case, the pattern branches out in multiple directions that stem from agent $i$.
    	If we trace any branch, because only a finite number of agents $N_{des}$ exists, we have the two following possible situations:
    	\begin{enumerate*}
    		\item The branch eventually features an agent $j$ with state $s_j\in\mathcal{S}_{simplicial}$.
	    	In the extreme, this is the leaf of the pattern.
	    	Here, we can have two situations:
    		\begin{enumerate*}
    			\item $s_j\in \mathcal{S}_{des}\cap\mathcal{S}_{simplicial} $.
    			If this exists, then the simplicial agent is also static. 
    			Therefore, it is possible that the entire pattern does not feature \emph{any} active and simplicial agent.
    			\item If $\mathcal{S}_{des}$ cannot, by design, form the clique of a state in $\mathcal{S}_{blocked}$, then it is guaranteed that $s_j\not\in\mathcal{S}_{des}\cap\mathcal{S}_{simplicial}$.
    		\end{enumerate*}
	    	Therefore, we can \emph{locally} impose that situation (b) always occurs, that situation (a) never occurs, and we thus guarantee that $s_j\in\mathcal{S}_{active} \cap \mathcal{S}_{simplicial}$ 
        (this is the first condition of this Lemma).

    		\item If all branches only feature non-simplicial states, then this is only explained if the branches form loops, otherwise at least one leaf would be present as in situation 1 above.
    		However, it can be ensured that a loop will always collapse and feature one simplicial active agent.
    		In a loop, all agents have two cliques, each formed by one neighbor.
    		$G_\mathcal{S}^{2r}$ tells whether any static agent with two neighbors, by the action of its neighbors, will become active.
    		If this is the case for all states, then we know that the action of any neighbor will cause a chain reaction about the loop which will eventually cause the loop to collapse about one corner point and create a simplicial leaf.
        This is the second condition of this Lemma.
        The collapse of two exemplary loops is depicted in \figref{fig:loopcollapse}.
    	\end{enumerate*}
    	In summary, by creating the conditions such that situation 1(a) never occurs, we restrict the possible patterns that can exist outside of $G_P^{AS}$ to patterns with only loops (situation 2).
      If $P_0$ is a loop, then through $G_\mathcal{S}^{2r}$, we know that loop patterns will collapse into a pattern that exists within $G_P^{AS}$.
      Else, $P_0$ is not a loop and it already exists within $G_P^{AS}$.
    	This means that any pattern $P_0$ will either exist within $G_P^{AS}$, or will transition into $G_P^{AS}$.
    \end{proof}
   
    \begin{figure}[t!]
    \centering
      \includegraphics[width=0.9\textwidth]{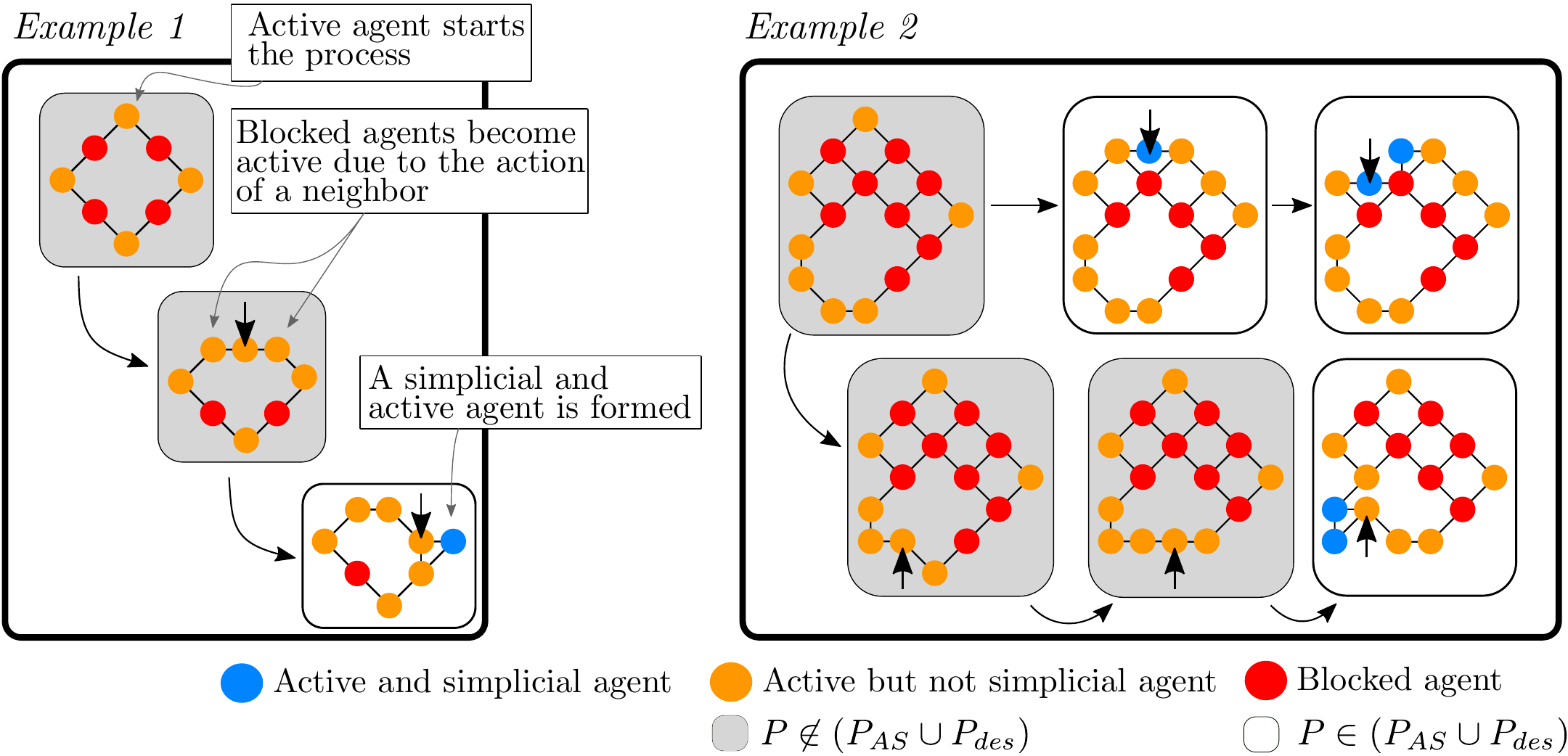}
      \caption{Illustration of two exemplary loops that ``collapse''. Notice that the active states present at the borders cause a chain reaction until eventually a simplicial active agent is present.
      This is a property that can be determined by inspecting $G_\mathcal{S}^{2r}$, which will show that the static agents will become active and propel the chain reaction.}
      \label{fig:loopcollapse}
    \end{figure}

    \begin{figure}[ht]
     \centering
      \begin{subfigure}[t]{0.19\textwidth}
        \centering
        \includegraphics[width=\textwidth]{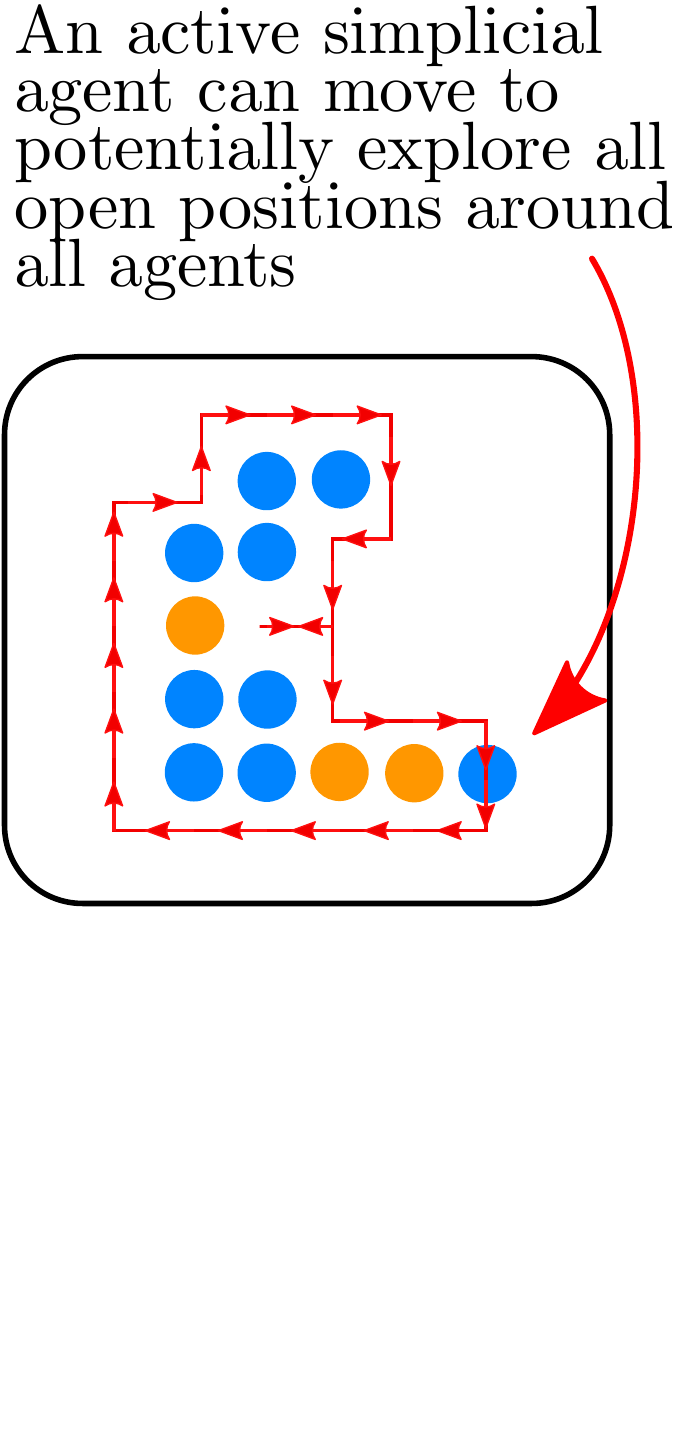}
        \caption{Simplicial agent that can travel to all open positions in the pattern.}
        \label{fig:allopenpositions}
      \end{subfigure}
      \hspace{5mm}
      \begin{subfigure}[t]{0.74\textwidth}
     \includegraphics[width=\textwidth]{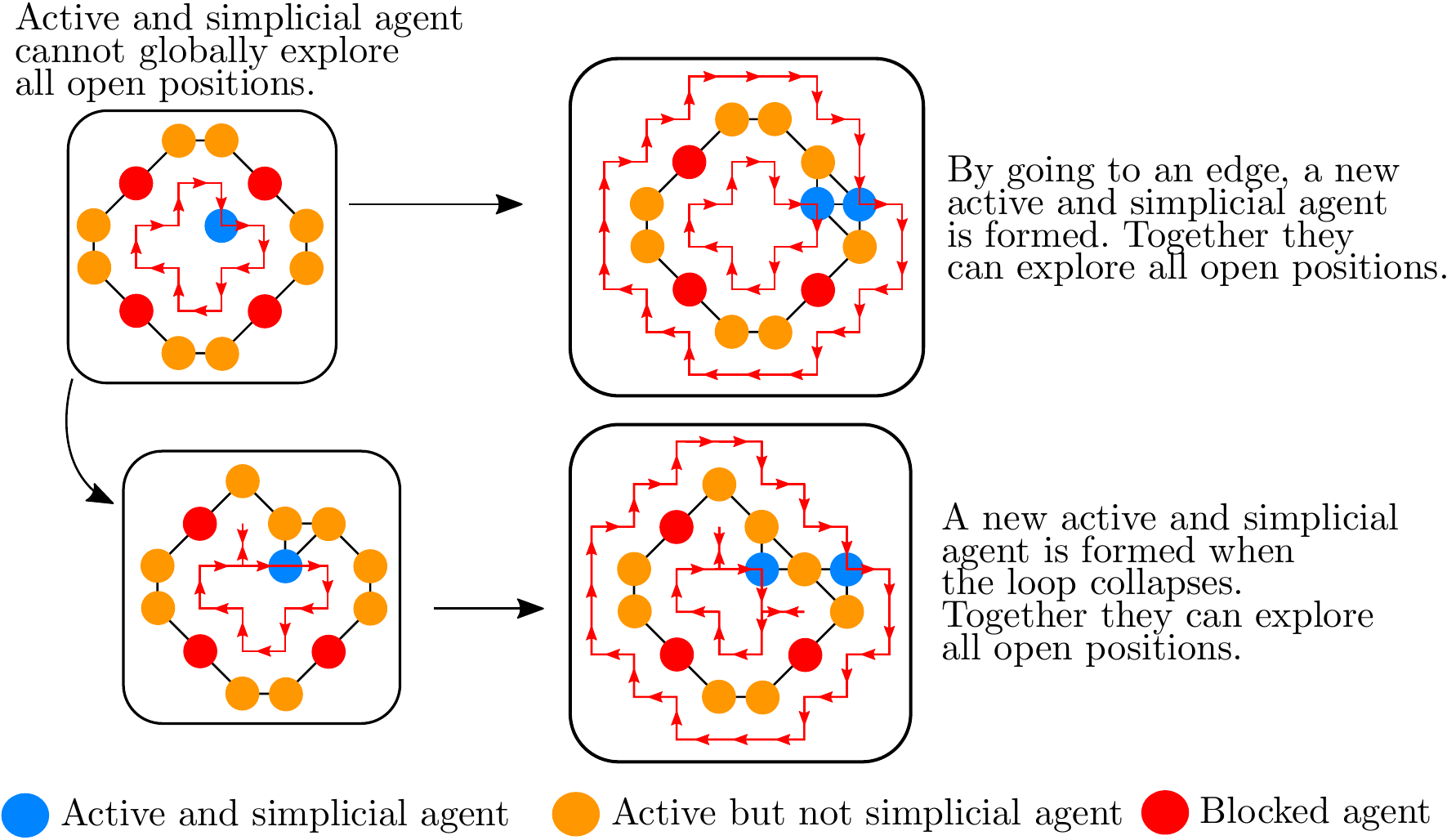}
      \caption{Two possibilities for how, should the agent be globally surrounded in a loop and unable to travel to all open positions, then this a new simplicial and active agent will take over.}
      \label{fig:loopcollapse_agentstuck}
      \end{subfigure}
      \caption{Illustration of how active and simplicial agents can travel to all open positions in the structure}
    \end{figure}

    With the conditions from Lemma \ref{lemma:activeandsimplicialpresent} we ensure that a simplicial active agent will always be present regardless of $P_0$.
    We now introduce Theorem \ref{theorem:p0pdes}, which we use to determine that $P_{des}$ will eventually form from $P_0$ and eliminate the chance for any deadlocks.

   \begin{theorem}
      \label{theorem:p0pdes}
      If the following conditions are satisfied:
      \begin{enumerate*}
      	\item $P_{des}$ is achievable,
      	\item $G_P^{AS}$ can be reached from any initial pattern in $G_P$,
	      \item $G_\mathcal{S}^1$
	      shows that any agent in state 
	      $s\in \mathcal{S}_{active}\cap\mathcal{S}_{simplicial}$
	      can move to explore all open positions surrounding its neighbors,
	      \item 
	      $G_\mathcal{S}^3$
	      shows that any agent in any state 
	      $s\in\mathcal{S}_{static}$ 
	      will always, by the arrival of a new neighbor in an open position, transition into an active agent (with the exception of any agent that is, or becomes, surrounded),
	    \end{enumerate*}
      then $P_{des}$ can be reached from any initial pattern $P_0$.
    \end{theorem}
    \begin{proof}
      In the following, we will show that any pattern in $G_P^{AS}$ will keep transitioning until it forms the desired pattern $P_{des}$.
    	Consider a swarm of $N_{des}$ agents arranged in a pattern $P_0$.
    	If $P_{des}$ is achievable, via Lemma \ref{l:achievability}, it can be reached by the actions of the agents, meaning that the node $P_{des}$ is in $G_P$
      (this is the first condition in this theorem).
	    Through Lemma \ref{lemma:activeandsimplicialpresent} we know that we can always get to a condition where one active and simplicial agent is present, such that we are in the graph $G_P^{AS}$ 
      (this is the second condition in this theorem).
	    We observe the case where at least one agent, agent $i$, exists with state $s_i\in\mathcal{S}_{active}\cap\mathcal{S}_{simplicial}$.
	    As agent $i$ moves, one of the following events can happen:
    	\begin{enumerate*}
    		\item Agent $i$ enters a state 
    		$s_i'\not\in\mathcal{S}_{simplicial}$.
        Via Lemma \ref{lemma:activeandsimplicialpresent}, at least one other agent is (or will be) in state
    		$s\in\mathcal{S}_{active}\cap\mathcal{S}_{simplicial}$.

    		\item Agent $i$ enters a state 
    		$s_i'\in\mathcal{S}_{static} \cap \mathcal{S}_{simplicial}$.
			  If $P_{des}$ is not yet achieved, then at least one other agent in the swarm is in an active state (Lemma \ref{lemma:active_states_present}).
    		If the active agent(s) are in state $s\in\mathcal{S}_{active} \cap \mathcal{S}_{\lnot simplicial}$, then this takes us back to point 1 in this list.
    		If the active agent(s) are in state $s\in\mathcal{S}_{active} \cap \mathcal{S}_{simplicial}$, 
    		this takes us to point 3 in this list.

    		\item Agent $i$, and/or agent(s) taking over, keeps moving and each time enters a state $s_i'\in\mathcal{S}_{active} \cap \mathcal{S}_{simplicial}$.
    		Via $G_\mathcal{S}^1$ we know that it can potentially explore all open positions surrounding all its neighbors (this is the third condition of this theorem).
    		As it moves, its neighbors also change, such that it always can potentially explore all open positions around all agents, and thus \emph{all open positions in the pattern} (see \figref{fig:allopenpositions} for a depiction).
        This means that the swarm can evolve towards a pattern that is closer to the desired one.
    	\end{enumerate*}
    		Therefore, any situation will always develop into the situation of point 3.
    		This is free of deadlocks, as all possible deadlock situations are mitigated:
  		\begin{enumerate*}
  
  			\item It may happen that the simplicial and active agents cannot actually visit all open positions in the swarm because, at the global level, it is enclosed in a loop by the other agents.
    		By Lemma \ref{lemma:activeandsimplicialpresent}, the loop will always collapse, meaning that at least one active simplicial agent will be freed, or that a new active simplicial agent will form.
        The new agent will be able to travel to all positions external to the loop, avoiding a deadlock.
        This is depicted in \figref{fig:loopcollapse_agentstuck}.
    		
    		\item Agent $i$ can travel about all open positions in the swarm as expected.
  			Via $G_\mathcal{S}^3$, we can extract that this must cause at least one static agent to become active (this is the fourth condition of this theorem).
  			Consider a static agent $j$ which becomes active when $i$ becomes its neighbor.
    		This may lead to one of the following developments, all of which avoid deadlocks.
    		\begin{enumerate*}
	    		\item Agent $i$ remains active and simplicial.
	    		The pattern can evolve even further and a deadlock is trivially avoided.
	    		\item Agent $i$ becomes static upon neighboring agent $j$, prior to the departure of agent $j$.
	    		In this case, either it will be freed by the departure of agent $j$, taking us back to point 2(a) in this list, or else it will remain static following the departure of agent $j$, taking us to point 2(c) in this list.
	    		\item Agent $i$ becomes static upon neighboring agent $j$, following the departure of agent $j$.
	    		It is now agent $j$ that can explore all open positions in the swarm, and further continue the process elsewhere.
	    		It is not possible that agent $j$ can uniquely come back to its original position, because by analysis of $G_\mathcal{S}^3$ we know that agent $j$ can free any static agent in the swarm, and not just agent $i$.
	    		\item Agent $i$, while agent $j$ is moving, enters the position
	    		(and state) that was originally taken by agent $j$.
	    		As in point 2(c) in this list, it is not possible that agent $j$ only frees agent $i$ in the same way that agent $i$ freed agent $j$, because $G_\mathcal{S}^3$ shows that agent $j$ can free any agent in the swarm, and not just agent $i$.
	    	\end{enumerate*}
	    	There is an exception to the rule, which are static states that either are, or become, surrounded by other agents.
	    	In this case, $G_\mathcal{S}^3$ may not show that they can become free.
	    	However, it is trivially impossible (since there is a finite number of agents) for the swarm to only feature agents that are surrounded.
        A situation where \emph{all} agents are \emph{all} surrounded cannot occur; at least one agent will not be surrounded.
        This justifies the exception to the fourth condition in this theorem.
	    \end{enumerate*}
  		With the above it is confirmed that
  		1) any open position in the pattern can potentially be filled, and 
  		2) no deadlocks will arise.
  		This means that the swarm will keep evolving into all achievable patterns.
  		Therefore, \emph{any} pattern in $G_P^{AS}$, including $P_{des}$, will \textemdash{ }given infinite time \textemdash{ }always eventually be formed starting from any other pattern in $G_P$.
    \end{proof}

    In this section, we presented a local proof that ensures that the desired pattern will be reached.
    We showed that, by ensuring a set of local conditions, we can determine that the pattern will be achieved from any initial configuration of the swarm.
    One of the main conditions is the need for simplicial active states, which brings interesting insights.
    The dependence on the set $\mathcal{S}_{des}$ leads to limitations on the desired patterns that may be reached independently of $P_0$.
    We note the following:
    \begin{itemize*}
      \item Desired states with only one neighbor may violate the first condition of Lemma \ref{lemma:activeandsimplicialpresent}.
      This is because this desired state can form the clique of a blocked state on its own.
      If this occurs, the local proof presented here is too restrictive to guarantee that the desired pattern will be formed without deadlocks.
      \item Removing a dependency on North (Assumption A\ref{a:north}) may lead to violating the first condition of Lemma \ref{lemma:activeandsimplicialpresent}.
      This is because states become rotation invariant, as discussed further in \secref{sec:discussion_thenorthdependency}.
    \end{itemize*}

   As it stands, the proof rests on the assumption that the desired pattern $P_{des}$ is the only pattern that can be formed where all agents are in a static state.
   In the following section, we will present a method to check for an arbitrary set of static states that a desired pattern is indeed the unique pattern in which all agents are static.

\section{Checking if the emergent pattern is unique}
\label{sec:implementation}

  There is a need to assess whether a swarm of robots with a set of static states $\mathcal{S}_{static}$ will uniquely form the desired pattern $P_{des}$.
  In this section we detail our implementation to check this for an arbitrary set $\mathcal{S}_{static}$.
  We focused on the case where all agents have omni-directional sensors, for which there exist several simplifying assumptions that enable a fast reduction of the search space.
  For a swarm of $N$ agents with a set $\mathcal{S}_{static}$ consisting of $\abs{\mathcal{S}_{static}}=N_S$ states, the states could coexist in $n_c$ combinations, where
  \begin{equation}
    n_c=\frac{(N_S+N-1)!}{N!(N_S-1)!}.
    \label{eq:nc}
  \end{equation}
  For each combination of states, there (may) exist multiple ways in which the states could be organized.
  To filter possible combinations and spatial arrangements,
  we used the implementation shown in \figref{fig:implementation}.
  We first assess whether a combination of states is viable.
  If so, we check for all different spatial arrangements of the states whether a spanning tree can exist with no loose edges (an edge where two or more states do not match).
  If there are no loose edges, then we examine the pattern to see if it is $P_{des}$.
  If this is not the case, then a counter-example has been found.
  If a counter-example is found, it means that the sensor layout is insufficient to guarantee that the desired pattern is the unique pattern that will be formed.
  
  \begin{figure}[ht]
    \centering
    \includegraphics[width=\textwidth]{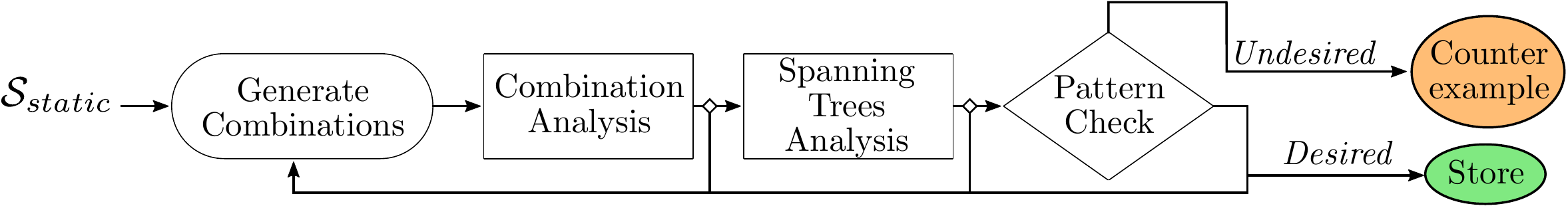}
    \caption{Diagram of automatic checker that checks whether $\mathcal{S}_{static}$, for a fixed number of agents $N$, can only form the desired pattern.}
    \label{fig:implementation}
  \end{figure}

  \subsection{Preliminaries}
  \label{sec:matchmatrices}
    Consider a set $\mathcal{S}_{static}$ with $\abs{\mathcal{S}_{static}}=d$.
    We introduce two tools to describe how any pair of states in $\mathcal{S}_{static}$ can be matched:
    the
    \textbf{Link-Direction matrix}
    (Definition \ref{def:linkdirectionmatrix}) and the
    \textbf{Match matrix}
    (Definition \ref{def:matchmatrix}).
    \begin{definition}
      \label{def:linkdirectionmatrix}
	  The \textbf{Link-Direction matrix} $D$ is a square matrix ($d\times d$) that holds the links (Definition \ref{definition:link}) along which any two states in $\mathcal{S}_{static}$ match (Definition \ref{def:match}).
    \end{definition}

    \begin{definition}
      \label{def:matchmatrix}
      The \textbf{Match matrix} $M$ is a matrix that holds the \emph{number of} links (Definition \ref{definition:link}) along which any two states in $\mathcal{S}_{static}$ match (Definition \ref{def:match}).
	  For omni-directional sensors, $M$ is symmetrical.
      Intuitively, this is because if agent $i$ sees agent $j$, then agent $j$ can also see agent $i$.
    \end{definition}

  \begin{figure}[t]
    \centering
    \includegraphics[width=0.8\textwidth]{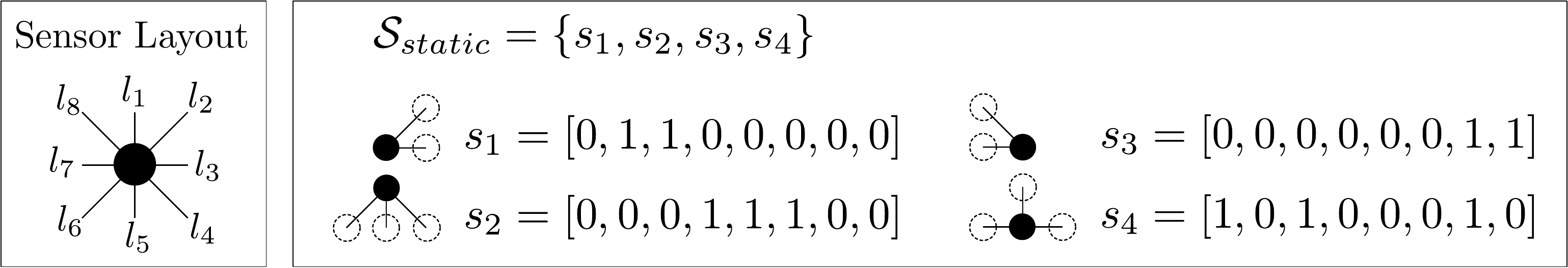}
    \caption{Arbitrary set $S_{static}$ used for examples in \secref{sec:matchmatrices} and \secref{sec:combinationanalysis}}
    \label{fig:ch5example}
  \end{figure}

    \paragraph{Example}
      Consider the set $\mathcal{S}_{static}=\{s_1,s_2,s_3,s_4\}$ depicted in \figref{fig:ch5example}.
      For this set:
      \begin{align}
      \renewcommand{\arraystretch}{1.3}
        D (\mathcal{S}_{static})= 
        \begin{Bmatrix}
           -  &  [l_2] &  -  & [l_3] \\
          [l_6] &   -  & [l_4] & [l_5] \\
           -  &  [l_8] &  -  & [l_7] \\
          [l_7] &  [l_1] & [l_3] &  -  
        \end{Bmatrix}
        &&
      \renewcommand{\arraystretch}{0.9}
      M(\mathcal{S}_{static}) = 
        \begin{bmatrix}
          0  & 1  & 0  & 1  \\
          1  & 0  & 1  & 1  \\
          0  & 1  & 0  & 1  \\
          1  & 1  & 1  & 0  \\
        \end{bmatrix}
        \nonumber
      \end{align}
      Note that all 0 entries in $M(\mathcal{S}_{static})$ correspond to empty entries in $D (\mathcal{S}_{static})$.
      From $M(\mathcal{S}_{static})$ we can quickly extract that state $s_1$ can never connect to itself, but it can connect to states $s_2$ and $s_4$.
      With $D (\mathcal{S}_{static})$ we can see that $s_1$ can match with $s_2$ along $l_2$, and with $s_4$ along $l_3$.
      Note that $D (\mathcal{S}_{static})$ matrix, although not strictly symmetric, also has a symmetry to it: each link always features, at its symmetry position, a link along the opposite direction.
      For example, if $s_1$ matches with $s_2$ along direction $l_2$, then $s_2$ matches with $s_1$ along $l_6$.\\

      If $\mathcal{S}_{static}$ includes a state that does not match with any state in the set, this will be seen as a null row in $M(\mathcal{S}_{static})$.
      If it can match, then $D (\mathcal{S}_{static})$ will show whether all its links can be matched by the other states in $\mathcal{S}_{static}$.
      If these requirements are not met, then the state can be excluded from analysis since it can never coexist with the other states.

  \subsection{Combination Analysis}
  \label{sec:combinationanalysis}

    \paragraph{Completeness Test}
    \label{sec:combination_completenesstest}
      A complete combination satisfies the following conditions:
      \begin{enumerate*}
        \item \textbf{The graph is complete.}
          Each link in any one direction should have a link in the opposite direction that will match it.
          Furthermore, any valid combination should consist of an even number of agents expecting an odd number of neighbors \citep{vansteen2010graph}.
        \item \textbf{The pattern is finite along all directions.}
          For each direction, there should be at least one state that does not require a link along that direction.
        \item \textbf{The edges of the pattern exist.}
          For each direction, there must be at least one state that features a link in that direction, but not in the opposite direction.
      \end{enumerate*}

    \paragraph{Matching Test}
      Each state in a combination should be capable of being matched by the other states in a combination.
      This information is provided by the Match Matrix.
      The reasoning is best explained via an example.
        Consider, for a swarm of 5 agents with $\mathcal{S}_{static}$ as the example in \secref{sec:matchmatrices}/\figref{fig:ch5example}, a potential combination
        $C_i = \{ s_1, s_1, s_2, s_3, s_3\}$.
        Using $M(\mathcal{S}_{static})$, we observe pair-wise matches that are possible between the states in $C_i$.
        $M$ tells us that $s_1$ can \emph{only} connect to $s_2$ in one direction.
        However, $C_i$ features \emph{two} instances of $s_1$ and \emph{only one} instance of $s_2$.
        This means that one instance of $s_1$ can never be satisfied; the combination can never exist.
        Furthermore, there should be enough states that match to $s_2$ in order to accommodate all its links.
        If there are too few states that match $s_2$, then we know that $s_2$ can never be satisfied in full, and the combination is also not valid.

  \subsection{Spanning Trees Analysis}
  \label{sec:permutationtest_1}
    If a combination passes the combination analysis, we construct and test spanning trees to determine whether and how the states could form a pattern.
    Spanning trees graphs are used here as a convenient tool to express how a pattern expands through space starting from an arbitrary root node.
    Let $T_i(C_k)$ represent an arbitrary spanning tree generated from a combination $C_k$.
    The nodes of $T_i$ are the states in $C_k$, and the edges of $T_i$ are one of the links between the states.
    With the tests below, we first test higher level properties of a generated spanning tree.
    If these properties are met then we test the spanning tree spatially to determine if all states match in full.
    Examples of trees that fail or pass the conditions are shown in \figref{fig:spanning_tree}.

    \begin{figure}[t]
      \centering
        \includegraphics[width=0.75\textwidth]{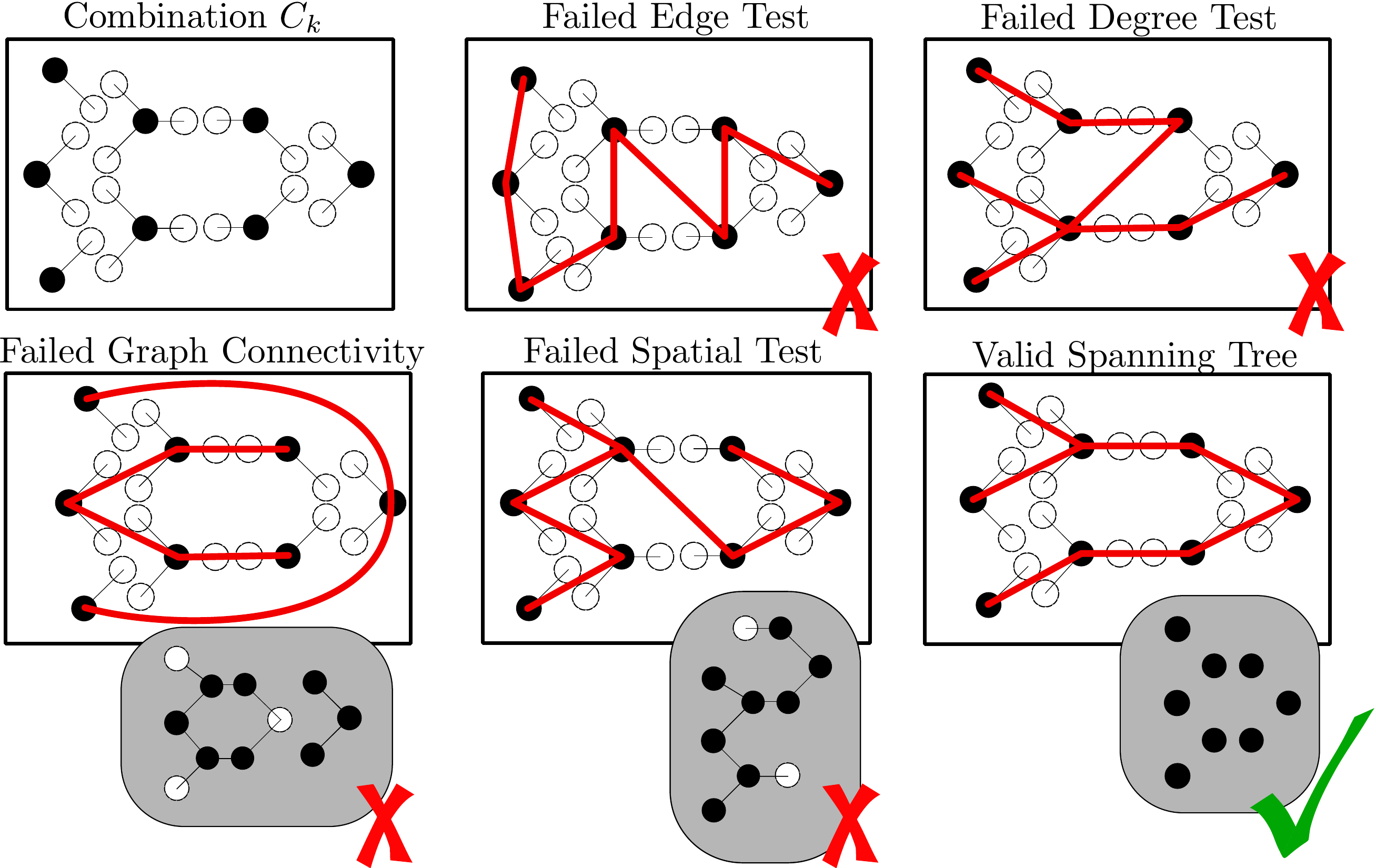}
      \caption{Representation of a valid spanning tree used to describe a pattern}
      \label{fig:spanning_tree}
    \end{figure}

    \begin{enumerate*}

    \item \textbf{Graph Edge test.}
      If $M(\mathcal{S}_{static})$ shows that any of the edges in $T_i(C_k)$ cannot exist (because a match between those states does not exist), then $T_i(C_k)$ is invalid.
      These spanning trees can be discarded.

    \item \textbf{Degree test.}
     The degree of state should be less than or equal to the number of links that the state holds.
      If the degree of a node in $T_i(C_k)$ is larger than the degree of the state, then $T_i(C_k)$ is invalid.
      These spanning trees can be discarded.

    \item \textbf{Compression.}
      If a combination features the same state multiple times, multiple spanning trees that created from the combination, such as $T_i(C_k)$ and $T_j(C_k)$, can be duplicates of each other.
      They are duplicates because the agents connected have the same state, but because all agents are homogeneous and anonymous this does not matter.
      All duplicate spanning trees can be ignored and only one needs to be analyzed.

    \item \textbf{Graph Connectivity.}
      It has been established that the swarm cannot disconnect, meaning that any pattern must have a connected spanning tree.
      If $T_i(C_k)$ is not connected, then it is invalid.
      
    \item \textbf{Spatial test.}
      Spanning trees that meet all conditions above are plotted in space and checked to make sure that all states match in full without loose ends.
      $D(\mathcal{S}_{static})$ can be used to quickly generate the full pattern.
    \end{enumerate*}

  \subsection{Pattern Check}
  \label{sec:traceanalysis}
    If a valid spanning tree is identified, a possible pattern has been found.
    The pattern needs can be checked to determine whether it is equal to the desired pattern $P_{des}$.
    A variety of methods can be used to do so automatically \citep{loncaric1998survey}.
    In our work, we used Fourier descriptors for plane closed curves to examine the contour of the pattern \citep{zahn1972fourier} and check against the contour of the desired pattern.

\section{Discrete space and discrete time simulations}
\label{sec:gridsimulations}
  In this section, the generation of different patterns by swarms is demonstrated and evaluated together with an exploration on how a further adaptation of the behavior may speed up the convergence to a desired pattern.
  The latter leads to insights on possible optimization strategies, which will be discussed in \secref{sec:discussion_optimization}.

  \subsection{Simulation environment and test description}
    The agents exist in an unbounded two dimensional grid world.
    The sensor layout is omni-directional and extends up to the nearest grid points, mimicking Case 2 in \figref{fig:state_space_design_options} or the example from \figref{fig:ch5example}.
    The agents can take action omni-directionally as seen in \figref{fig:action_space_design_options}.
    For the purposes of this simulation, to ensure that the swarm fully abides to Proposition \ref{proposition:oneatatime}, only one agent moves at any given time step (this is an assumption that will be lifted in the next section).
    At each time step, one random active agent in the swarm performs an action and moves to a new grid point.
    All tests begin by initializing the agents in a random formation and are repeated 100 times.
    We explored the formation of:
      \begin{inparaenum}[1)]
        \item a triangle with 4 agents,
        \item a triangle with 9 agents, and
        \item a hexagon with 6 agents
      \end{inparaenum}
    under the following behaviors:
    \begin{itemize*}
    \item \emph{Baseline}:
      At each time step, a random agent with state $s\in\mathcal{S}_{active}$ is selected and executes a random safe action based on $\mathcal{Q}_{safe}$.

    \item \emph{Alteration 1 (ALT1)}:
      same as baseline; however, when an agent moves at time-step $k$, the same agent will not move at time-step $k+1$ (unless it is the only active agent).

    \item \emph{Alteration 2 (ALT2)}:
      same as ALT1; additionally, all states with more than 5 neighbors are now included in $\mathcal{S}_{blocked}$.

    \item \emph{Alteration 3 (ALT3)}:
      same as ALT2; additionally, all actions must ensure that all agents in the neighborhood, following the action, have at least one neighbor at North, South, East or West, else the action is discarded from $\mathcal{Q}_{safe}$.
      There is only one exception to this, and it is the state
          $s = \begin{bmatrix} 1 & 0 & 1 & 0 & 0 & 0 & 1 & 0 \end{bmatrix}$, for which otherwise a spurious pattern was found following the procedure in \secref{sec:implementation}.

    \item\emph{Alteration 4 (ALT4)}:
      same as ALT3; additionally, all states with more than 4 neighbors are now also included in $\mathcal{S}_{blocked}$.
    \end{itemize*}

    The motivation behind the different behaviors is to explore which parameters may have an influence over how many steps it takes to form the pattern, on average.
    The reasoning behind ALT3 and ALT4 is to force the agents to ``cut-corners'', as well as to give the agents less actions to choose from.
    ALT3 and ALT4 are such that the desired stated that compose the hexagon cannot be achieved.
    Therefore, based on Lemma \ref{l:achievability}, the hexagon should not be achievable by these controllers.
    We further note that ALT3 and ALT4 also do not meet condition 3 of Theorem \ref{theorem:p0pdes}, because some active simplicial agents are prevented from exploring all open positions surrounding their neighbors.
    However, as discussed in \secref{sec:pursuitofhappiness}, the local conditions in Lemma \ref{lemma:activeandsimplicialpresent} and Theorem \ref{theorem:p0pdes} are more restrictive than required and do not necesserily apply to all global patterns.
    We will also use these simulations to explore this point.

  \subsection{Results}
  	Distributions for the number of steps to completion are shown in \figref{fig:idealizedsimulations_shapeanalysis_steps}.
  	For the Baseline, ALT1, and ALT2, the final pattern is achieved in all tests.
    As the size of the pattern grows, ALT1 is seen to provide for a better performance.
    This is explained by the fact that it limits the possibility that an agent cycles back and forth between two spots, which is inherently inefficient.
    ALT2 further improves the results; blocking all states with 5 neighbors reduced the size of $G_P$ such that the swarm had less patterns to explore.
    ALT3 and ALT4 further reduced $G_P$, leading to significant boosts in performance.
    However, as expected through Lemma \ref{l:achievability}, ALT3 and ALT4 did not work on the hexagon configuration \textemdash{ }the hexagon pattern did not emerge and the swarm continued randomly reshuffling.
    This gives empirical confirmation of Lemma \ref{l:achievability} and also provides practical insight into how tuning the state-action space can be beneficial for some patterns, but detrimental to others.
    We also note that, even though condition 3 of Theorem \ref{theorem:p0pdes} was not met by ATL3 and ATL4, they still managed to achieve the pattern in all cases.
    This shows that the local proof, as presented in this paper, can be too restrictive and needs to be inspected further if one wishes to optimize the performance of the system (alternatively, it could also be possible that the agents were simply ``lucky'' to not encounter deadlock situations during any of our simulations).

    \begin{figure}[t]
      \centering
      \begin{subfigure}[t]{0.18\textheight}
        \centering
        \includegraphics[height=0.2\textheight]{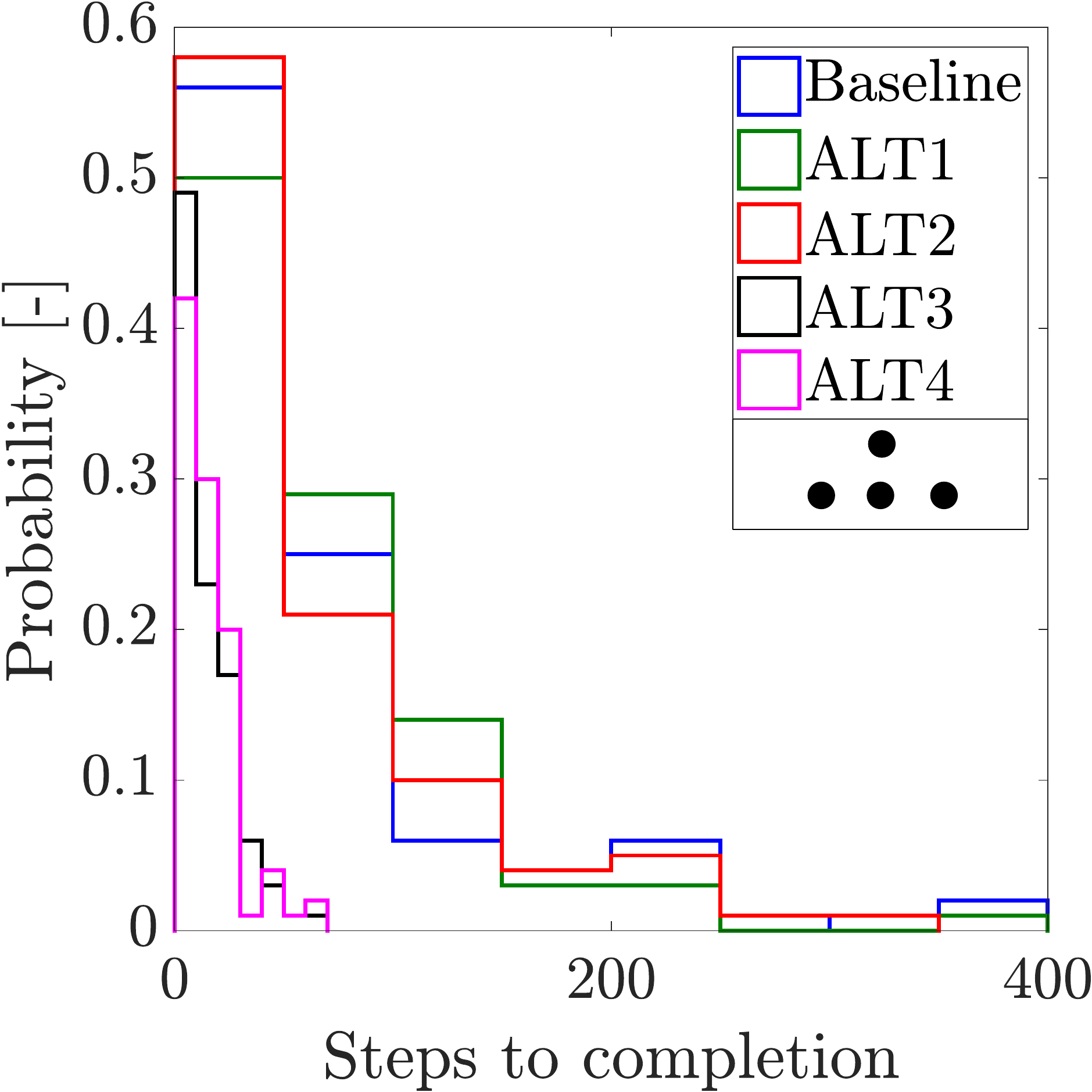}
        \caption{4 agent triangle}
        \label{fig:idealizedsimulation_shapeanalysis_steps}
      \end{subfigure}
      \hspace{5mm}
      \begin{subfigure}[t]{0.18\textheight}
        \centering
        \includegraphics[height=0.2\textheight]{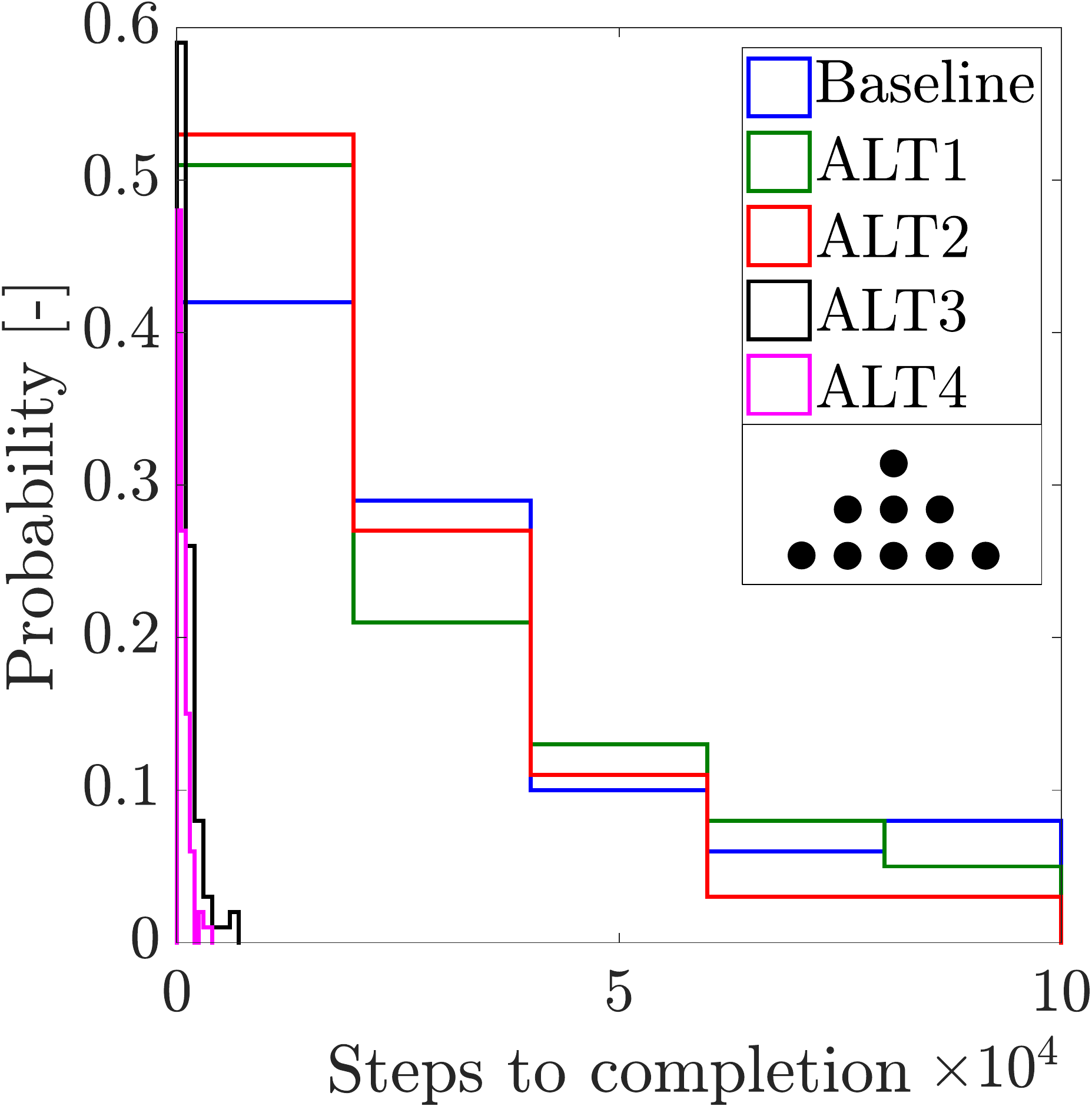}
        \caption{9 agent triangle}
        \label{fig:idealizedsimulation_shapeanalysis_steps}
      \end{subfigure}
      \hspace{5mm}
      \begin{subfigure}[t]{0.18\textheight}
        \centering
        \includegraphics[height=0.2\textheight]{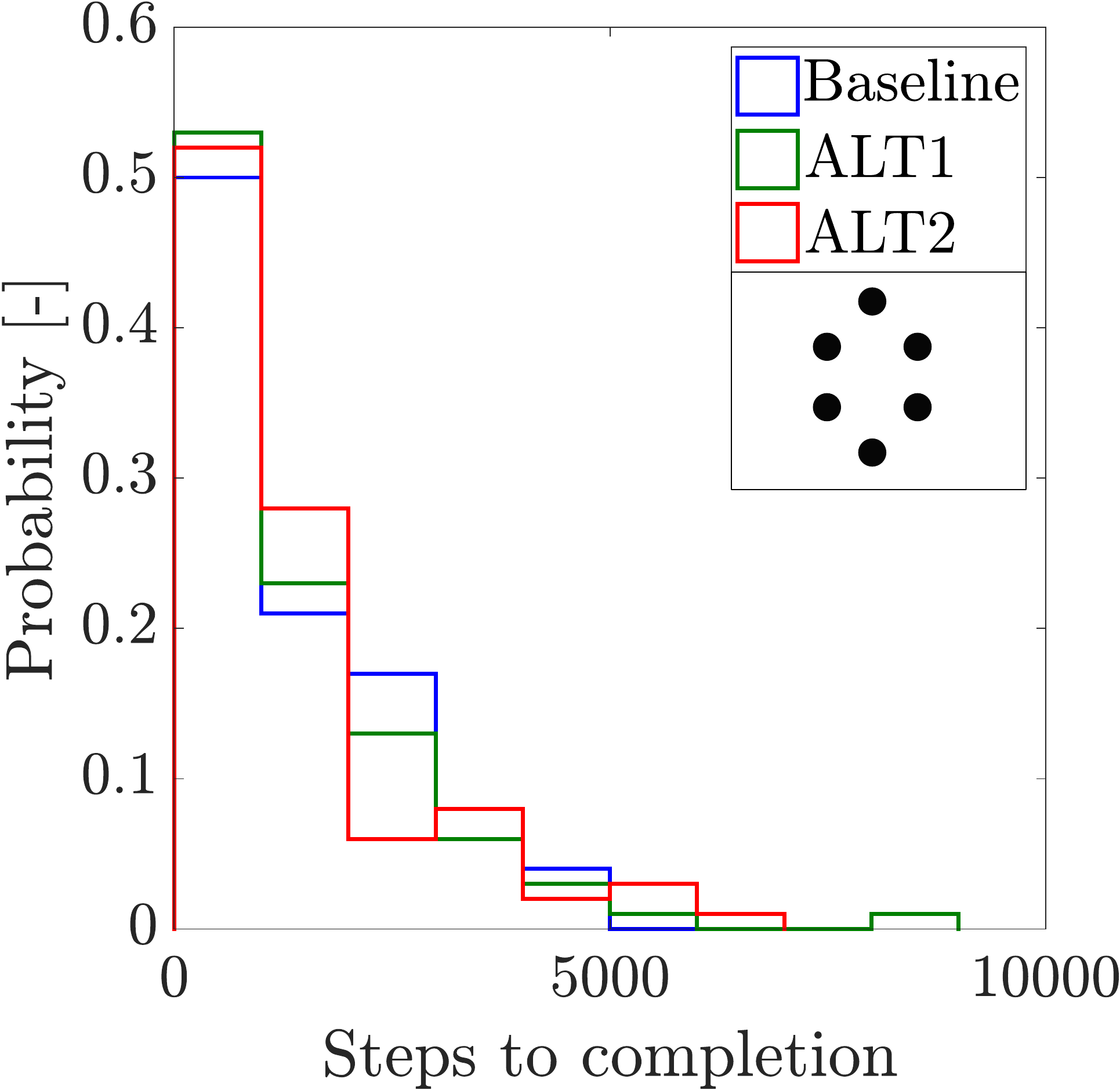}
        \caption{Hexagon}
        \label{fig:idealizedsimulation_shapeanalysis_steps}
      \end{subfigure}
      \caption{Probability distributions of steps to completion by different state-action spaces for three tested patterns}
      \label{fig:idealizedsimulations_shapeanalysis_steps}
    \end{figure}

\section{Continuous space and continuous time simulations with asynchronous agents}
\label{sec:swarmulatorsimulations}
  The discrete space and time experiments from \secref{sec:gridsimulations} were ported to a continuous environment with asynchronous agents operating in continuous time.
  The aim was to see how well the system would port to a continuous, asynchronous setting, which is not accounted for in the proofs.

  \subsection{Simulation set-up and system description}
    
    \paragraph{Agent dynamics and behavior}
    The robots behave like accelerated particles, freely moving in an unbounded 2D space, and regulate their accelerations in a North-East frame of reference (we will denote $x$ for North, and $y$ for East).
    They can sense omni-directionally all their neighbors within a radius $\rho_{sensor}$, and they can sense the motion of their neighbors with enough accuracy to determine whether they are computing an action 
    (Assumption A\ref{a:canimove}).
    Each robot determines its local state in $\mathcal{S}$ following the barriers depicted in \figref{fig:simulationbreakdown}.
    When an agent is active and none of its neighbors are taking an action, it will try and take an action itself.
    Following an alignment maneuver of time $t_{{adj}_1}$, the agent will begin to take the action, moving with commanded speed $\abs{v_{action}}$.
    The agent interrupts the action if it senses another agent being too close or also performing an action, in an attempt to approach Proposition \ref{proposition:oneatatime}.
    If an action is completed or interrupted, the agent will perform a second alignment maneuver for time $t_{{adj}_2}$ to settle into its new position.
    Using $(t_{{adj}_1},t_{{adj}_2})>0$ we can instill the same behavior introduced in  ALT1 from \secref{sec:gridsimulations}, because the agent that has just taken action will not do so while adjusting, leaving a time window for its neighbors to take action.
    Note that alignment maneuvers are minimal and are thus not perceived by other agents as actions.
    Pseudo-code for the on-board controller of the agents is provided in \algref{alg:controller}.
    \begin{figure}[t]
      \centering
      \begin{subfigure}[t]{0.16\textheight}
        \centering
        \includegraphics[height=0.15\textheight]{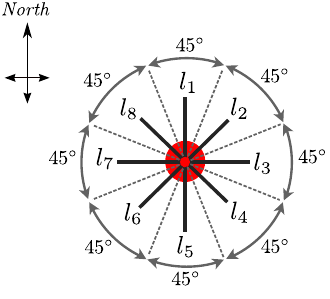}
        \caption{Rounding method used by the agents to assess their state}
        \label{fig:simulationbreakdown}
      \end{subfigure}
      \hspace{10mm}
      \begin{subfigure}[t]{0.14\textheight}
        \centering
        \includegraphics[height=0.15\textheight]{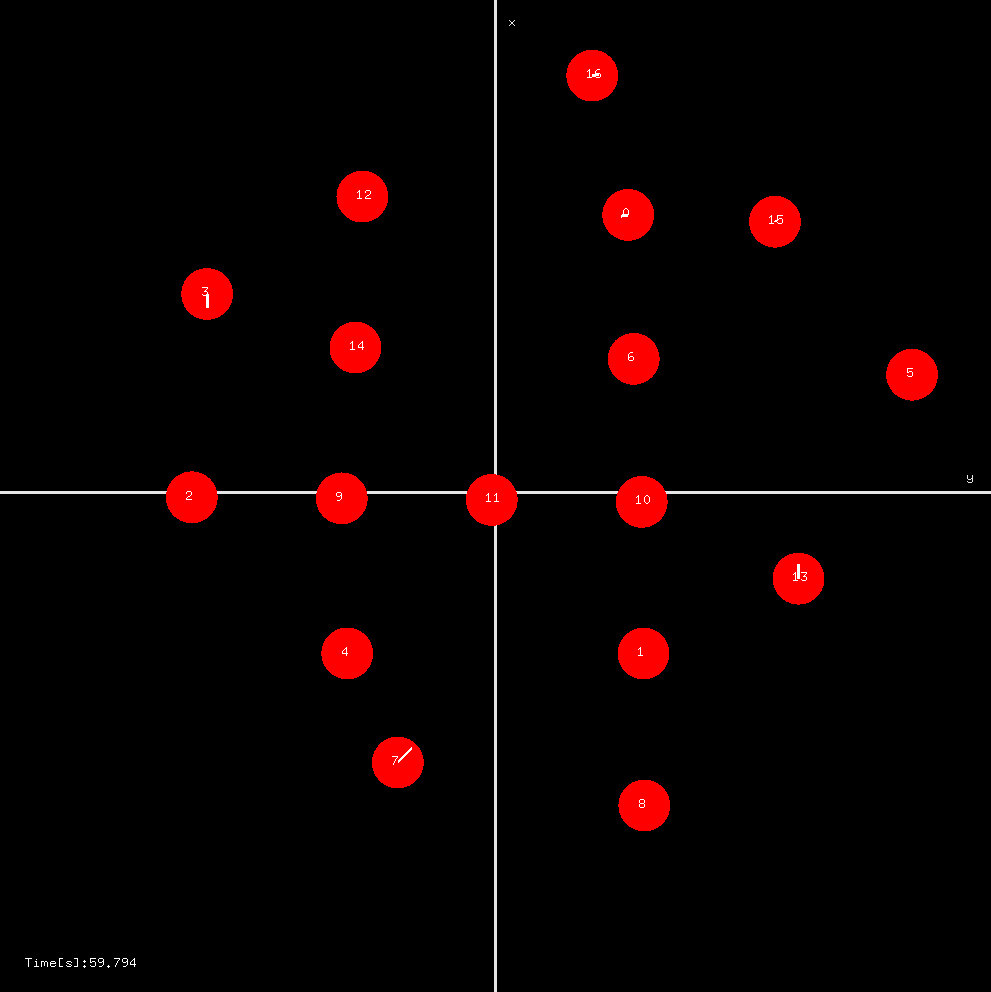}
        \caption{Several agents in action}
        \label{fig:simulationscreenshot}
      \end{subfigure}
      \hspace{10mm}
      \begin{subfigure}[t]{0.14\textheight}
        \centering
        \includegraphics[height=0.15\textheight]{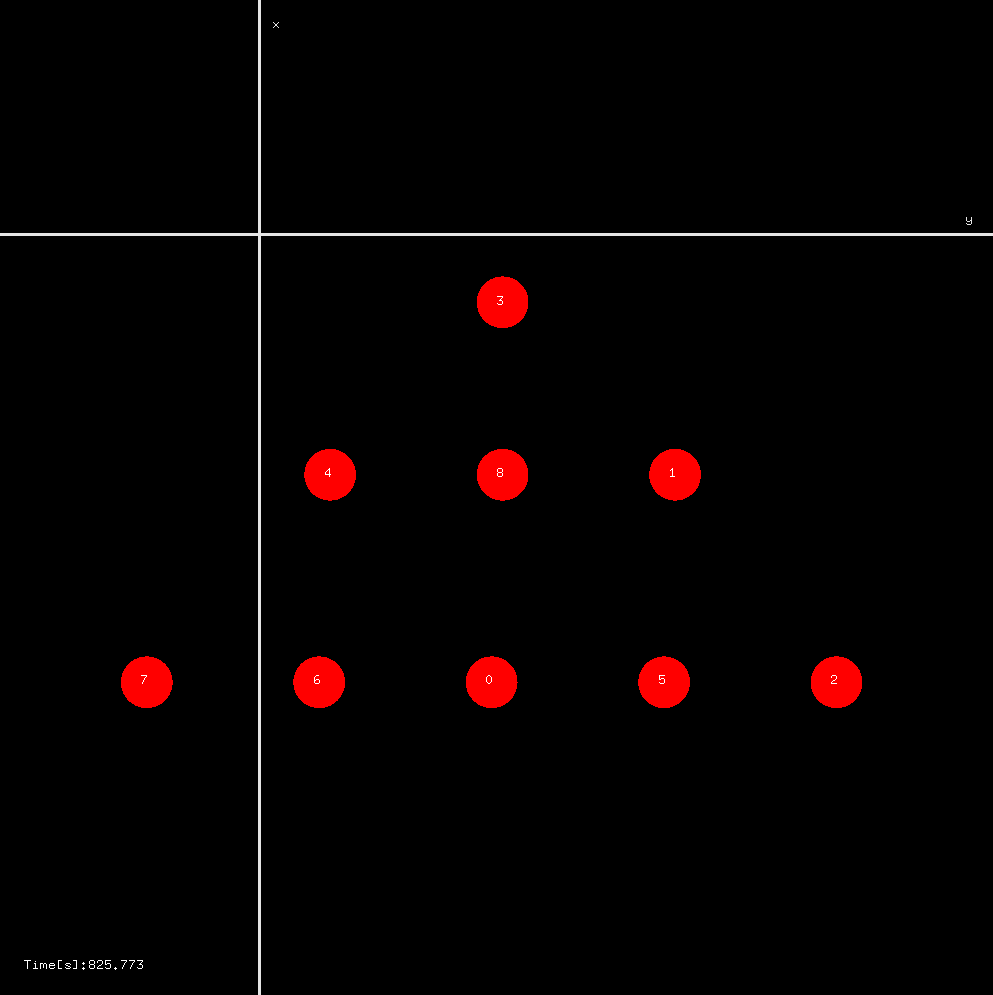}
        \caption{9 agents converged to a triangle}
        \label{fig:simulationscreenshot}
      \end{subfigure}
      \caption{State assessment (a) and two screenshots of continuous time and space simulations with asynchronous agents (b,c)}
      \label{fig:simulationfigures}
    \end{figure}

    \begin{algorithm}[h]
     \While{running}{
        Measure current relative positions of agents within $\rho_{sensor}$\;
        Determine discrete state $s$\;
        Determine whether any neighbors are taking an action\;

        \uIf{Not taking an action}{
          \uIf{All neighbors are not taking an action} {
            Adjust distance and bearing to closest neighbor(s) for $t_{{adj}_1}$\;
            \uIf{($s\in\mathcal{S}_{active}$) $AND$ (distance to all neighbors $>r_{min}$)} {
              Take an action from $\mathcal{Q}_{safe}$ with velocity $\abs{v_{action}}$\;
              Adjust distance and bearing to closest neighbor for $t_{{adj}_2}$\;
            }
          }
        }
        \uElseIf{(All neighbors are not taking an action) $AND$ (taking an action) $AND$ (distance to all neighbors $>r_{min}$)} {
          Continue action\;
        }
        \uElse {
          Stop action\;
        }
      }

     \caption{Pseudo-code for agent controller}
     \label{alg:controller}
    \end{algorithm}

    \paragraph{Distance and Bearing Adjustment Commands} 
      When not performing actions, the robots are governed by attraction-repulsion and alignment forces with respect to their neighbors.
      Consider two robots $\mathcal{R}_i$ and $\mathcal{R}_j$.
      The commanded velocity of $\mathcal{R}_i$ aligning to $\mathcal{R}_j$ along the $x$-direction (and, equivalently, $y$-direction) is given by
      \begin{equation}
        v_{{x}_{cmd_{ij}}} = (v_{r_{ij}}+v_{b_{ij}}) \cos(\beta_{ij}) - v_{b_{ij}} \cos(2\beta_{des}-\beta_{ij}).
        \label{eq:vxcmdvycmd_local}
      \end{equation}
      The first term is an attraction-repulsion term, and the second term is a bearing alignment term.
      Together, they cause $\mathcal{R}_i$ to gravitate to a specific distance and bearing ($\beta_{des}$) to its neighbor $\mathcal{R}_j$.
      $\beta_{ij}$ is the bearing to $\mathcal{R}_j$ with respect to North.
      $v_{b_{ij}}$ is a constant indicating the desired velocity of the bearing alignment.
      The velocity of the attraction-repulsion $v_{r_{ij}}$ is given by
      \begin{equation}
        v_{r_{ij}} = -k_r\frac{1}{\abs{\rho_{ij}}} + 
              \frac{1}{1+e^{-k_a(\abs{\rho_{ij}}-\rho_s)}},
        \label{eq:attrepown}
      \end{equation}
      where:
      $k_r \geq 0$ is a repulsion gain,
      $k_a \geq 0$ is an attraction gain,
      $\rho_{ij}$ is the distance that $\mathcal{R}_{i}$ measures to $\mathcal{R}_j$, and
      $\rho_s$ is a shift in the attraction term used to tune the equilibrium point to a desired distance $\rho_{des}$.
      For a given $\rho_{des}$, $k_r$, and $k_a$, one can extract $\rho_s$ for $v_{r_{ij}}=0$.
      \eqnref{eq:attrepown} has Lyapunov stability \citep{gazi2002class}.
      Two agents are in equilibrium
      ($v_{x_{{cmd}_{ij}}}=v_{y_{{cmd}_{ij}}}=v_{x_{{cmd}_{ji}}}=v_{y_{{cmd}_{ji}}}=0$)
      when $\beta_{ij}=\beta_{des}$, $\beta_{ji}=\beta_{des}\pm\pi$, and $v_{r_{ij}}=v_{r_{ji}}=0$.
      Note how the alignment between agents is reciprocal;
      such that for each $\beta_{des}$, there is also a corresponding $\beta_{des}\pm\pi$ alignment.
      This is due to the identities 
      $\sin(\beta+\pi) = -\sin(\beta)$ and 
      $\cos(\beta+\pi) = -\cos(\beta)$ which manifest themselves via \eqnref{eq:vxcmdvycmd_local}.
      Multiple alignment bearings $\beta_{des}$ can be defined, we then allow the agent to select the one that is closest to $\beta$.
      For a robot $\mathcal{R}_i$ which senses $m$ neighbors, the full alignment command in $x$ is $v_{x_{{cmd}_i}} = \sum_{j=1}^{m}v_{{x}_{cmd_{ij}}}$, and the equivalent for $y$.
      This is unless the closest neighbor is at a distance $\rho < \rho_{safe}$, where otherwise only that agent is considered.
      
    \paragraph{Simulation Parameters}
      In our tests we used:
        $\rho_{sensor}=1.6$m,
        $\rho_{des}=1$m,
        $\rho_{safe}=0.3$m,
        $t_{{adj}_1}=2$s,
        $t_{{adj}_2}=9$s,
        $k_r=1$,
        $k_a=5$,
        $v_{action}=1$m/s,
        $v_b = 10$m/s.
      The state-action set $\mathcal{Q}_{safe}$ and the active set $\mathcal{S}_{active}$ were as in ALT4 from \secref{sec:gridsimulations}.
      We provided the agents with $\beta_{des} = \{0,\pi/4,\pi/2,3\pi/4\}$, making them adjust at all bearings to each other that match the state space.
      For $\beta_{des} = \pi/4$ and $\beta_{des} = 3\pi/4$, then we define $\rho_{des}=\sqrt{2}$m instead of $\rho_{des}=1$m.
      The adjacency matrix describing how the swarm is connected is continuously computed, and a Breadth-First Search (BFS) is used to check that the swarm remains connected.
      If the swarm disconnects at any point, the simulation exits.
      Alternatively, the simulation exits once the desired pattern is achieved.
      For each pattern, simulations are repeated 50 times.

    \begin{figure}[t]
      \centering
      \begin{subfigure}[t]{0.45\textwidth}
        \centering
        \includegraphics[width=\textwidth]{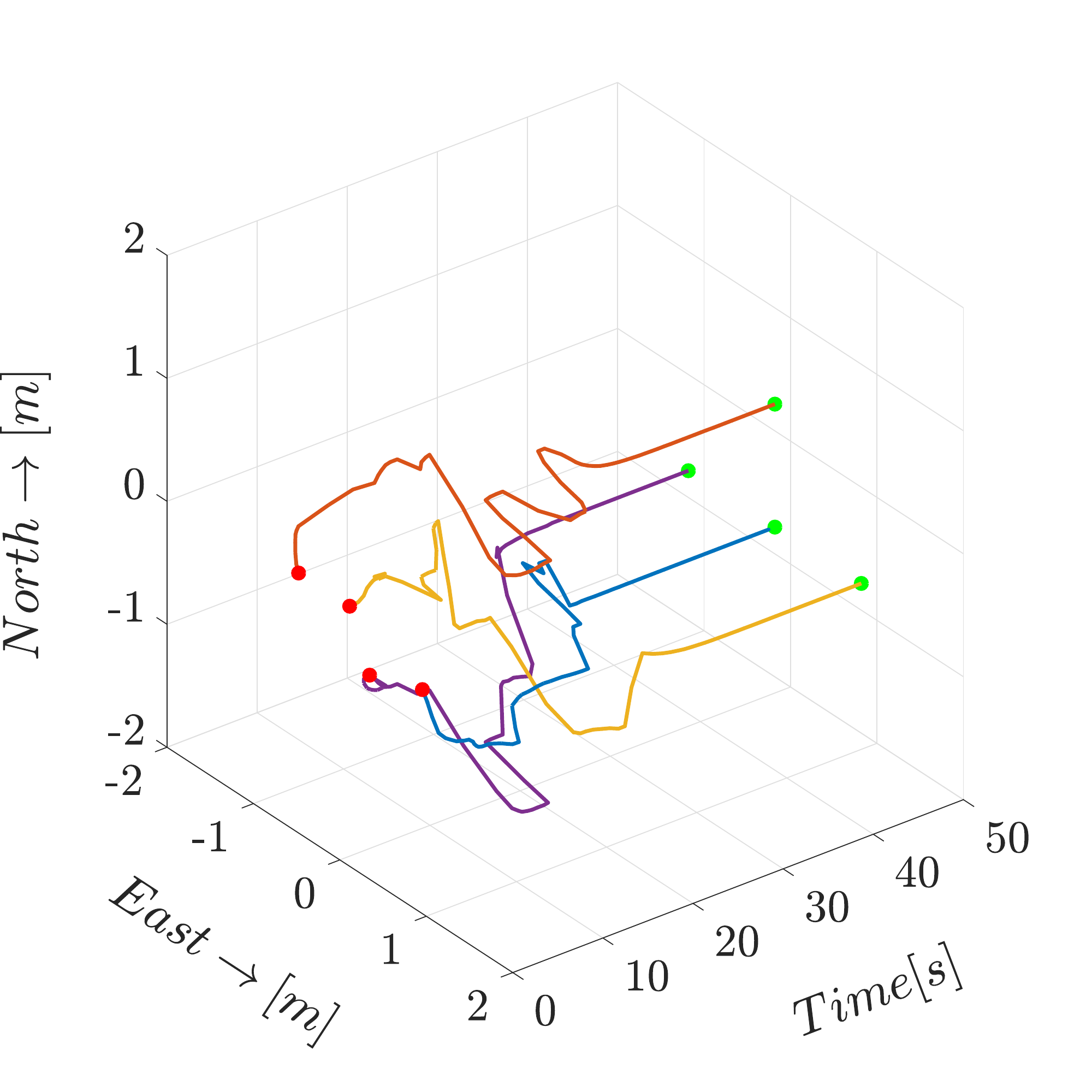}
        \caption{Triangle with 4 agents}
        \label{fig:swarmulatortrajectory_triangle4}
      \end{subfigure}
      \hspace{5mm}
      \begin{subfigure}[t]{0.45\textwidth}
        \centering
        \includegraphics[width=\textwidth]{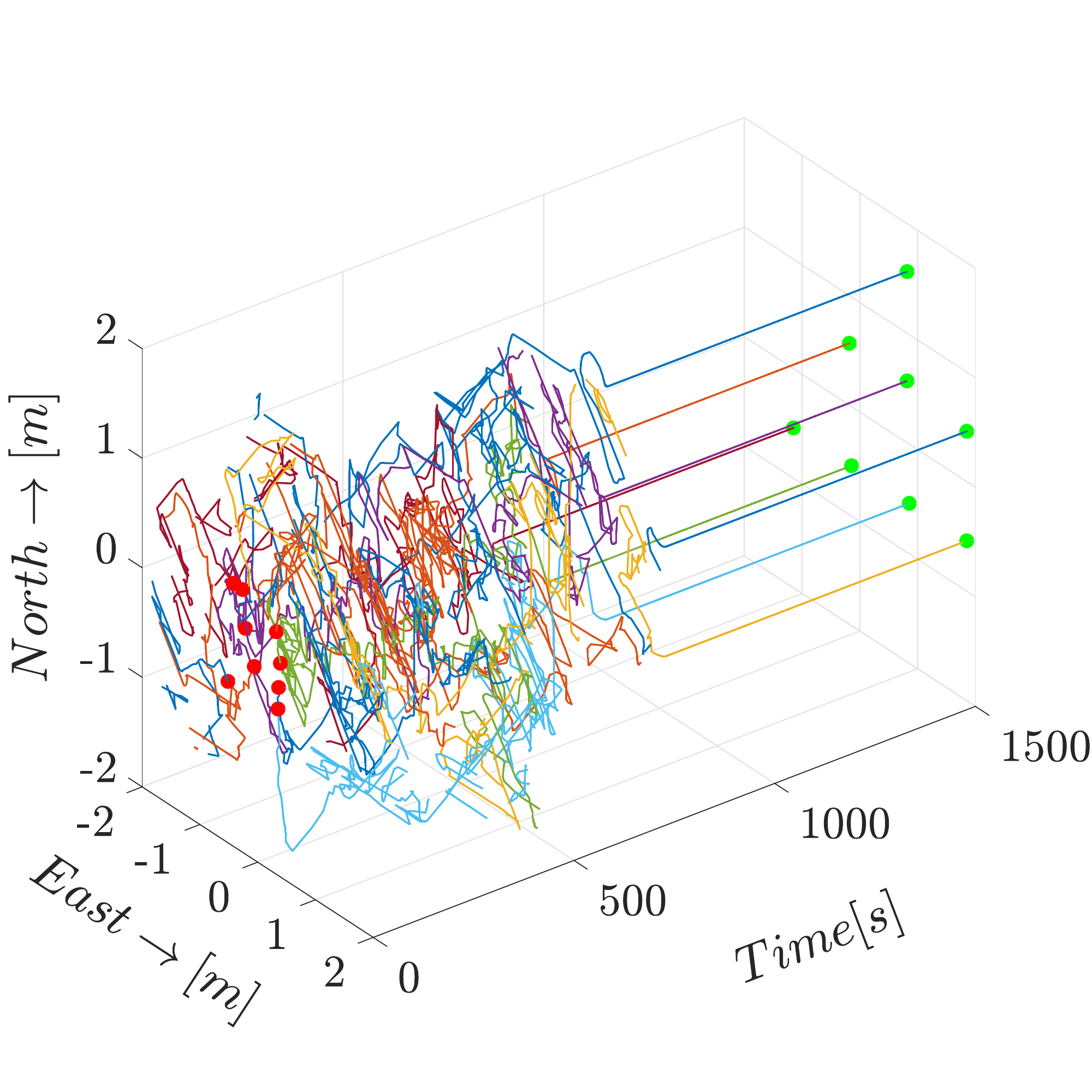}
        \caption{Triangle with 9 agents}
        \label{fig:swarmulatortrajectory_triangle9}
      \end{subfigure}
      \caption{Simulated trajectories to the desired patterns}
      \label{fig:swarmulatortrajectories}
    \end{figure}

    \begin{figure}[t]
      \centering
      \begin{subfigure}[t]{0.40\textwidth}
        \centering
        \includegraphics[width=\textwidth]{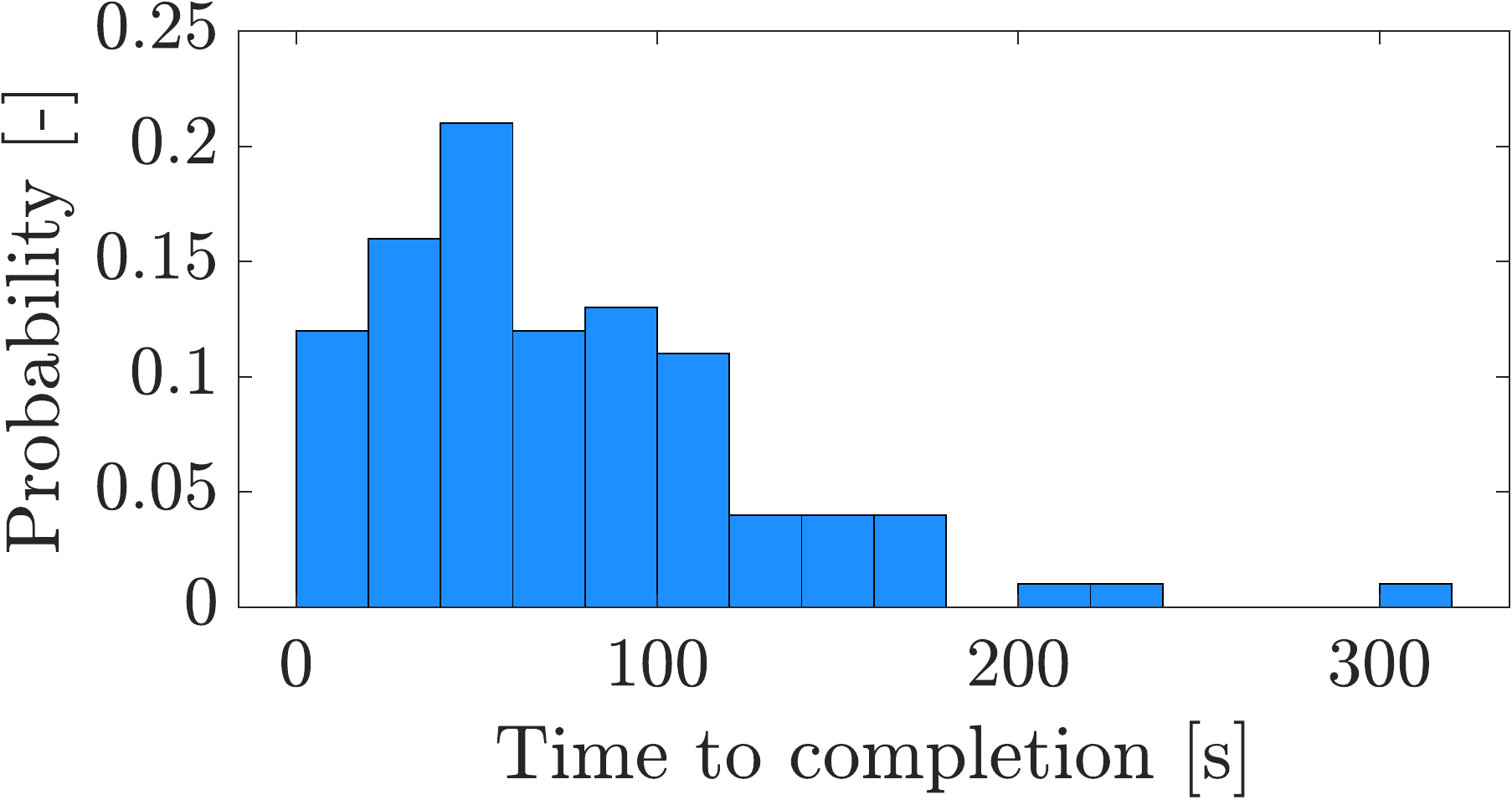}
        \caption{Triangle with 4 agents, bin width=$20$s}
        \label{fig:swarmulatortimes_triangle4}
      \end{subfigure}
      \hspace{5mm}
      \begin{subfigure}[t]{0.40\textwidth}
        \centering
        \includegraphics[width=\textwidth]{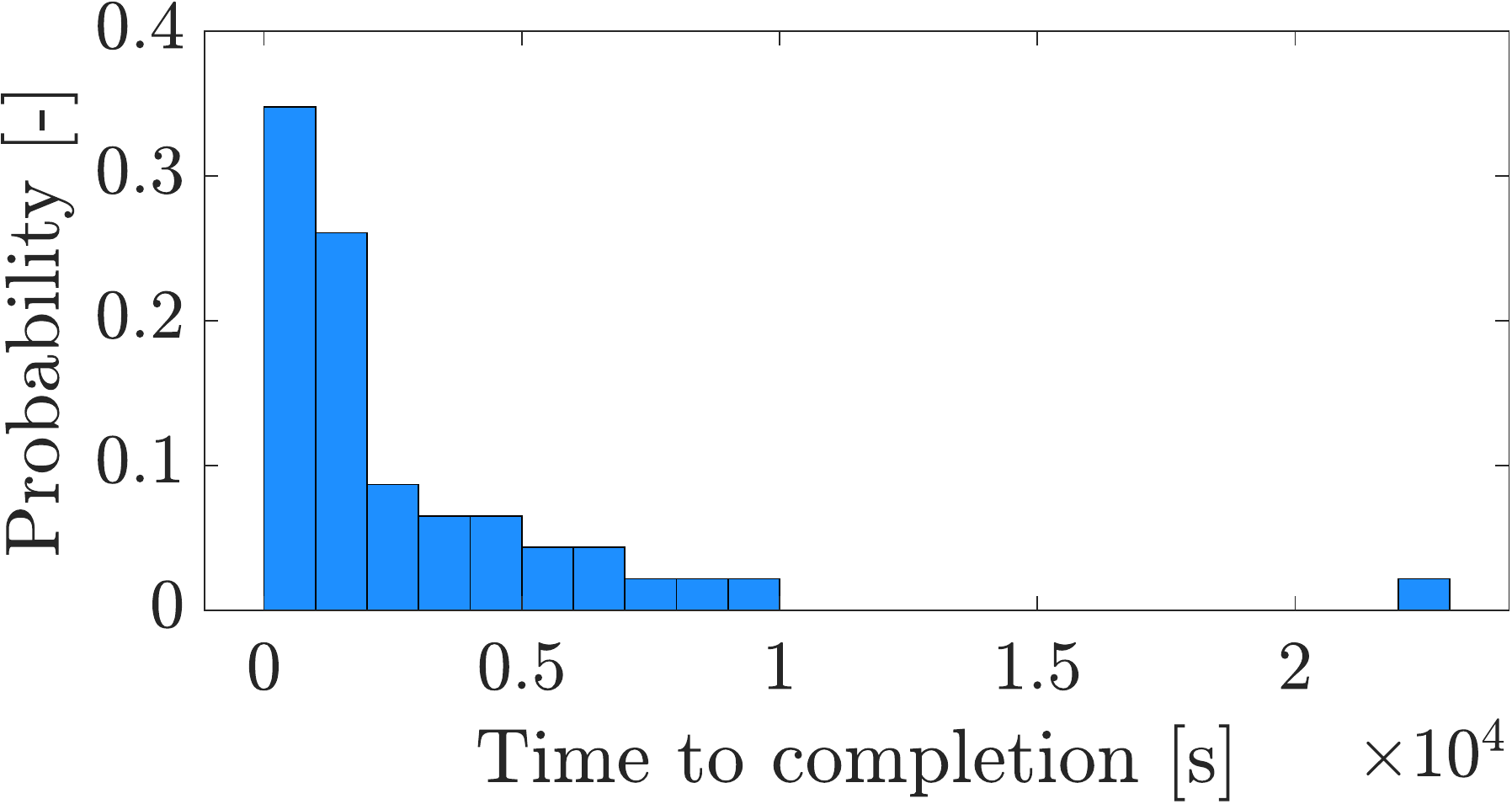}
        \caption{Triangle with 9 agents, bin width=$1000$s}
        \label{fig:swarmulatortimes_triangle9}
      \end{subfigure}
      \caption{Probability density for time to convergence based on simulation results}
      \label{fig:swarmulatortimes}
    \end{figure}
  \subsection{Simulation Results}
    The results for the triangles with 4 and 9 agents from \secref{sec:gridsimulations}, using the controller from ALT4, were validated in this continuous setting.
    \figref{fig:swarmulatortrajectories} shows sample trajectories over time.
    We can see that the agents reshuffle until the desired pattern is achieved.
   	The triangle with 4 agents was achieved successfully in 50 out of 50 trials, with generally fast convergence times (within 100 seconds of simulated time).
   	The triangle with 9 agents was achieved successfully in 49 out of 50 trials.
   	One trial experienced a separation due to the violation of the condition in Proposition \ref{proposition:oneatatime} which caused an unsafe maneuver.
   	This happened as two non-neighboring agents chose to perform an action at approximately the same time, came into each other's view, yet the alignment maneuver was such that two agents (who were the link between two parts of the swarm) moved further than $1.6$m apart, which was the limit of the sensor.
   	Although we could expect the swarm to reconnect, the issue is noted and should be tackled in future work.
   	Nevertheless, this was \emph{the only} unsafe maneuver that took place out of thousands of maneuvers executed over all 50 trials.
   	The times to completion and their probability are shown in \figref{fig:swarmulatortimes}.

\section{Discussion}
\label{sec:discussion}

  \subsection{Insight into emergent behavior of swarms}
    The approach presented in this work offers novel insights into how emergence for a fully distributed swarm with very limited agents can be achieved.
    In analogy to biological systems, the agents in the swarm merely functioned on the instincts to:
    \begin{enumerate*}
        \item \emph{be safe} (not collide with others); 
        \item \emph{be social} (not risk separation from/of the group); 
        \item \emph{be happy} (be in a set of desired local states).
    \end{enumerate*}
    The final pattern emerged as the only unique combination of ``happy'' states.
    The agents had no knowledge that the desired local states were only a piece of a larger pattern, nor did they need to care.
    This shows that an emergent global swarm behavior can be reached merely by breaking it down into its locally observable constituents.
    Similarly, the framework can be used to guarantee the \emph{lack of} emergence in case all static states cannot be arranged into any pattern. \\

    The principles in this work were aimed at pattern formation, but future work could aim to extend to other applications, such as organized navigation or task allocation.
    Additionally, the formal methodology to avoid collisions and avoid separations also has several applications on its own.
    For instance, it can be used to guarantee that a wireless sensors network would never separate in multiple groups even when faced with obstacles.
    This proof is independent of the number of agents that are added or removed to/from the system, and has empirically been shown to work in a continuous time and space realm.

  \subsection{State Explosion}
  	Our approach uses a local level analysis to determine whether a unique pattern will be achieved by the agents starting from any initial condition.
  	This is independent of the number of agents in the swarm, and it is thus free of state explosion issues.
  	Nevertheless, the approach as a whole still requires us to determine whether a desired pattern is unique, and this part was still done using a global analysis as in \secref{sec:implementation}.
  	With our implementation we aimed to mitigate a computation explosion by catching unfeasible patterns as early as possible, yet the issue remains.
    In future research, there should be efforts to further mitigate its effects for finite patterns.\\

    Three solution avenues have been identified for this problem.
    The first avenue is to focus on the agents at the border of the structure, assuming that all other agents will be enclosed by these agents.
    The second avenue is to use repeating patterns.
    The local states could be made such that the agents can arrange into infinitely repeating patterns (e.g. infinitely connecting hexagons) and create a large complex structure without defining the larger structure in full.
    The third solution avenue is to allow blocked agents that have been blocked for a long time to (temporarily) perform unsafe maneuvers, which might set the system in motion (but may come at the cost of the swarm being disconnected).

  \subsection{Achieving Complex Patterns in a Short Time}
  \label{sec:discussion_optimization}
    The results in \secref{sec:gridsimulations} indicate that, although the patterns will be achieved, it can take a significant amount of steps.
    However, our results also show that the behaviors can be tuned in order to improve performance by several orders of magnitude.
    We found that the tuning depends on the desired pattern (for instance, ALT3 and ALT4 could not be used to form a hexagon).
    This leads to the question: \emph{what is the optimum state-action mapping for a given pattern?}
    This problem could be solved using classical machine learning methodologies such as Reinforcement Learning (RL) or Evolutionary Robotics (ER).
    RL might not be well suited to the task since the agents only have partial knowledge of the environment, and are thus subjected to aliasing states \citep{kaelbling1996reinforcement}.
    Given that the system is non-Markovian, ER might be a better candidate \citep{decroon2005evolutionary}.
    In this context, the objective would be to determine the optimal alteration of the state-action mapping such that 
	  the agents still achieve their goal and 
	  their time to completion is minimized.
    To this end, the conditions expressed by Lemmas \ref{l:achievability}, Lemma \ref{lemma:activeandsimplicialpresent}, and Theorem \ref{theorem:p0pdes} allow to locally check for state-action relations that can reliably achieve the global goal.
    However, ALT3 and ALT4 have shown that certain conditions may be too restrictive.
    Future work should explore how far the restrictions can be lifted without affecting the local proof.
    Additionally, it might also be possible to improve convergence by adjusting waiting time, such that certain states wait longer than others before choosing to move, or by extending the sensor range, such that agents can take smarter actions towards desired states.
    Providing the agents with memory could be a further enhancement to the system \citep{mccallum1996reinforcement}.

  \subsection{The North Dependency}
  \label{sec:discussion_thenorthdependency}
    In this paper, we assumed all agents shared knowledge of a common direction (Assumption A\ref{a:north}).
    This can appear as a significant limitation, yet the framework presented in this paper can take absence of a North direction into account.
    The knowledge of a common direction is not essential, but it does enable an agent to differentiate between otherwise equivalent states.
    For example, consider the state 
    $s = \begin{bmatrix} 0 & 0 & 0 & 0 & 0 & 1 & 1 & 0 \end{bmatrix}$.
    Without a common direction:
    \begin{align}
      s &= \begin{bmatrix} 0 & 0 & 0 & 0 & 0 & 1 & 1 & 0 \end{bmatrix}
         \equiv \begin{bmatrix} 0 & 0 & 0 & 0 & 0 & 0 & 1 & 1 \end{bmatrix} \nonumber \\
         &\equiv \begin{bmatrix} 1 & 0 & 0 & 0 & 0 & 0 & 0 & 1 \end{bmatrix} \equiv \hdots \equiv 
         \begin{bmatrix} 0 & 0 & 0 & 0 & 1 & 1 & 0 & 0 \end{bmatrix}. \nonumber
    \end{align}
    Therefore, including a state in $\mathcal{S}_{static}$ would automatically include all rotations of that state as well.
    Even if the final pattern that is formed is still unique, neglecting North has still been shown to be a potential problem for the condition imposed by Lemma \ref{lemma:activeandsimplicialpresent}, which could be subject to local situations that cause the pattern to never emerge.
    However, if such cases are accepted (or otherwise circumvented) then Assumption A\ref{a:north} can be lifted.

\section{Conclusion and Future Work}
\label{sec:conclusion}
  This work introduced a method, complete with a proof procedure, to devise local behaviors of highly limited agents such that a unique global pattern always emerges.
  Approaching the problem from top-down, we first identify the local states that build the desired global pattern, and then check if these local states indeed lead uniquely to the desired global pattern or if the swarm can also converge to other undesired emergent solutions.
  If the desired pattern is a unique emergent solution, then we can locally prove, based on the agents' local behavior, whether the pattern is achievable and whether it can be reached without any issues from any initial pattern.
  Despite breaking down the system to a discrete state-space, and imposing the requirement in our proof that only one agent can move at once, we have shown that results can be reproduced by asynchronous agents operating in continuous time and continuous (unbounded) space.
  The methodology shown here has been used for agents in a two dimensional spatial plane.
  At its core, however, the methodology is based on the more general idea of matching local states to each other in order to synthesize a unique larger global state.
  With the correct mapping, we expect this strategy to also be applicable to systems with significantly different state and action spaces.\\

  Future work will focus on bringing this framework to real world robots.
  As pattern formation can take a long time if the agents take random actions, the first step will be to explore state-action optimization.
  This can be done either in the formal setting, by optimizing the state-action pairs directly, or in the dynamic setting, by studying how delays and waiting times in each state can affect the course of the structure.
  Optimization procedures should be such that, on average, the amount of steps to completion are minimized.
  A second step will be to explore the impact of noise and disturbances, which are inevitable in real world systems.
  A detailed study on the levels of noise and system errors that are deemed acceptable, and how to handle it, would be of very high importance.
  Of particular interest is the impact of false positives/false negative sensor readings, which may cause one of the agents to have a mistaken view of its local state.
  This may cause temporary heterogeneity in the system because that agent will not follow the rules as expected.
  This needs to be investigated. We expect that it should be possible to formally account for possible sensor errors by restricting the state-action mapping, although this could further restrict the possible patterns that can be formed while keeping the local proof intact.
  Finally, it would also be interesting to study the formation power that can be achieved when the assumption of north dependency is lifted, and the impact that this will have on the conditions of the local proof.

\bibliographystyle{spbasic}   
\bibliography{bibliography}   

\end{document}